\documentclass{article}

\usepackage[final]{ARLET_2024}

\usepackage[utf8]{inputenc} 
\usepackage[T1]{fontenc}    
\usepackage{hyperref}       
\usepackage{url}            
\usepackage{booktabs}       
\usepackage{amsfonts}       
\usepackage{nicefrac}       
\usepackage{microtype}      
\usepackage{xcolor}         
\usepackage[textwidth=3cm]{todonotes}

\usepackage{amsmath}
\usepackage{amssymb}
\usepackage{mathtools}
\usepackage{amsthm}

\usepackage{algorithm}
\usepackage{algorithmic}
\usepackage{enumerate}
\usepackage{mathrsfs}
\usepackage{bm}
\usepackage{float}
\usepackage{array}
\usepackage{dsfont}
\usepackage{comment}
\usepackage{natbib}
\usepackage{anyfontsize}
\usepackage{caption}
\usepackage{tikz}
\usepackage{pstricks}
\usepackage{setspace}
\usepackage{xspace}
\usepackage{pifont}
\usepackage{threeparttable}
\allowdisplaybreaks

\usepackage{thmtools}
\usepackage{thm-restate}

\theoremstyle{plain}
\newtheorem{theorem}{Theorem}[section]

\newtheorem{lemma}[theorem]{Lemma}
\newtheorem{corollary}[theorem]{Corollary}
\theoremstyle{definition}
\newtheorem{definition}[theorem]{Definition}
\newtheorem{example}[theorem]{Example}

\theoremstyle{remark}
\newtheorem{remark}[theorem]{Remark}

\DeclareMathOperator*{\argmin}{arg\,min}
\DeclareMathOperator*{\argmax}{arg\,max}

\def\1{\mathds{1}}
\def\E{\mathbb{E}}
\def\P{\mathbb{P}}
\def\R{\mathbb{R}}
\def\S{\mathcal{S}}
\def\A{\mathcal{A}}
\def\X{\mathcal{X}}
\def\C{\mathcal{C}}
\def\D{\mathcal{D}}
\def\bk{(k)}
\def\bkone{(k+1)}

\def\critic{\mathrm{critic}}
\def\actor{\mathrm{actor}}
\def\KL{\mathrm{KL}}

\def\algofullname{Dual Approximation Policy Optimization\xspace}
\def\algoacronym{DAPO\xspace}

\def\targ{\mathrm{targ}}
\def\AMPO{AMPO\xspace}

\newcommand{\mc}[1]{\mathcal{#1}}
\newcommand{\Sp}[1]{\left(#1\right)}
\newcommand{\Mp}[1]{\left[#1\right]}
\newcommand{\Bp}[1]{\left\{#1\right\}}
\newcommand{\abs}[1]{\left|#1\right|}
\newcommand{\Norm}[1]{\left\|#1\right\|}

\newcommand{\inner}[1]{\left\langle#1\right\rangle}
\newcommand{\matenv}[1]{\left[\begin{matrix}#1\end{matrix}\right]}
\newcommand{\dom}{\mathrm{dom}\,}
\newcommand{\intr}{\mathrm{int}\,}

\newcommand{\proj}{\mathrm{proj}}
\newcommand{\ones}{\mathbf{1}}
\newcommand{\mmid}{\ \middle|\ }

\newcommand{\revision}[1]{{\color{black} #1}}

\makeatletter
\newcommand{\IfTwoColumnElse}[2]{%
    \if@twocolumn
        #1 
    \else
        #2 
    \fi
}
\makeatother

\title{\algofullname}

\author{%
Zhihan Xiong\\
University of Washington\\
Seattle, WA 98195, USA \\
\texttt{zhihanx@cs.washington.edu} \\
\And
Maryam Fazel\\
University of Washington\\
Seattle, WA 98195, USA \\
\texttt{mfazel@uw.edu} \\
\And
Lin Xiao\\
FAIR at Meta\\
Seattle, WA 98109, USA \\
\texttt{linx@meta.com}\\
}

\begin{document}
\maketitle

\begin{abstract}
	We propose Dual Approximation Policy Optimization (DAPO), a framework that incorporates general function approximation into policy mirror descent methods. In contrast to the popular approach of using the $L_2$-norm to measure function approximation errors, DAPO uses the dual Bregman divergence induced by the mirror map for policy projection. This duality framework has both theoretical and practical implications: not only does it achieve fast linear convergence with general function approximation, but it also includes several well-known practical methods as special cases, immediately providing strong convergence guarantees.

\end{abstract}

\section{Introduction}
\label{sec:intro}

Policy gradient methods represent a paradigm shift in reinforcement learning from value-based methods \citep{watkins1989learning, puterman1994markov, bertsekas2015dynamic} to a more direct approach of policy optimization \citep{williams1992simple, sutton1999policy, konda1999actor}.
In particular, the natural policy gradient (NPG) method of \citet{kakade2001natural} inspired later development of trust region policy optimization (TRPO) \citep{schulman2015trust} and proximal policy optimization (PPO) \cite{schulman2017proximal}, both with great empirical success.

These successes ignited considerable efforts to understand policy gradient methods from a theoretical perspective. Among them, \citet{neu2017unified} first connected NPG with the mirror descent (MD) algorithm \citep{nemirovskij1983problem, beck2003mirror}, which led to a more general class of policy mirror descent (PMD) methods. 
Convergence guarantees for tabular PMD methods progressed from sublinear convergence \citep{shani2020adaptive, agarwal2021theory} to linear convergence \citep{xiao2022convergence, lan2023policy, johnson2023optimal}. 
Then the linear convergence results were extended to PMD methods with linear function approximation \citep{yuan2022linear}, and more recently with general function approximation \citep{alfano2023novel}. 

However, the progresses of PMD on the empirical and theoretical fronts are more or less disjoint, especially concerning general function approximation.  
One one hand, \citet{tomar2020mirror} and \citet{vaswani2021general} derived practical algorithms from the MD principle, but with no or limited convergence guarantees. 
On the other hand, \citet{alfano2023novel} proposed Approximate Mirror Policy Optimization (AMPO), a PMD framework that has  linear convergence guarantee with general function approximation, but has limited empirical success
(see our empirical study in Section~\ref{sec:experiments}).

In this paper, we aim to bridge this gap between theory and practice by proposing \algofullname (\algoacronym), a new PMD framework that incorporates general function approximation. 
In contrast to AMPO, which uses the squared $L_2$-norm to measure the function approximation error and tries to minimize it for policy update, \algoacronym uses the \revision{\textit{dual Bregman divergence}} generated by the mirror map used for policy projection.

We present several instantiations of \algoacronym using different mirror maps and prove linear convergence rates for two variants, \algoacronym-$L_2$ equipped with the squared $L_2$-norm as mirror map, and \algoacronym-KL with the negative entropy.
We show that \algoacronym-KL includes two state-of-the-art practical algorithms as special cases: 
Soft Actor-Critic (SAC) of \citet{haarnoja2018soft2} and Mirror Descent Policy Optimization (MDPO) of \citet{tomar2020mirror}, thus immediately providing them with strong convergence guarantees.
We compare \algoacronym with SAC and AMPO on several standard MuJoCo benchmark tasks to demonstrate the effectiveness of 
\revision{this duality framework}.

In addition, in order to work with negative entropy restricted on the simplex in the setting of general function approximation, we extend the MD theory to work with mirror maps whose gradient mapping and conjugate mapping are not inverses of each other, a technical contribution of independent interest. 



\section{Preliminaries}
\label{sec:prelim}

We first review the background of Markov decision processes (MDPs) 
and the general MD algorithm.

\subsection{Markov Decision Processes}
\label{sec:mdp_prelim}

Let $\Delta(\mc{X})=\Bp{p\in\R^{\abs{\X}}\mid \sum_{x\in\X}p_x=1 \text{ and } p_x\geq 0, \forall x}$ denote the probability simplex over an arbitrary finite set~$\mc{X}$. 
We consider an infinite-horizon Markov Decision Process (MDP), denoted as $\mc{M}=\Sp{\S, \A, \mc{P}, c, \gamma}$, where $\S$ is a finite state space, $\A$ is a finite action space, $\mc{P}:\S\times\A\mapsto\Delta(\S)$ is the transition kernel, $c:\S\times\A\mapsto[0, 1]$ is the single-step cost function and $\gamma\in(0,1)$ is the discount factor. A stationary policy is defined as a function $\pi:\S\mapsto\Delta(\A)$ such that $\pi_s$ is a probability distribution over $\A$ for each $s\in\S$. At each time $t$, an agent with policy $\pi$ takes an action $a_t\sim \pi_{s_t}$, which sends the MDP to the new state $s_{t+1}\sim\mc{P}(s_t, a_t)$ and
incurs a single-step cost $c(s_t,a_t)$. 

Our main objective is to find a policy that minimizes the accumulated, discounted cost starting from an initial state distribution $\rho\in\Delta(\S)$. 
Formally, it is defined as 
$V_{\rho}^{\pi}=\E_{s\sim\rho}\Mp{V^\pi_s}$, where
\begin{equation}\label{equ:def_V}
\textstyle
    V_s^\pi=\E_{a_t\sim \pi_{s_t}}\bigl[\,\sum_{t=0}^{\infty}\gamma^tc(s_t, a_t)\mid s_0=s\,\bigr].
\end{equation}
The corresponding Q-value function under policy $\pi$ and state-action pair $(s, a)$ is defined as
\begin{equation}\label{equ:def_Q}
\textstyle
    Q^{\pi}_{s, a}=\E_{a_t\sim \pi_{s_t}}\bigl[\,\sum_{t=0}^{\infty}\gamma^tc(s_t, a_t)\mid s_0=s, a_0=a\,\bigr].
\end{equation}
We use $Q^{\pi}_s\in\R^{\abs{\A}}$ to denote the vector $\Mp{Q^{\pi}_{s, a}}_{a\in\A}$ and we immediately have $V_s^{\pi}=\inner{Q^{\pi}_s, \pi_s}$.

With initial distribution $\rho\in\Delta(\S)$, we define the discounted state-visitation distribution under $\pi$ as
\begin{equation}\label{equ:def_visitation}
\textstyle
    d^{\pi}_{\rho, s}=(1-\gamma)\sum_{t=0}^{\infty}\gamma^t\,\P^{\pi}_{s_0\sim\rho}\Sp{s_t=s}, 
\end{equation}
where $\P^{\pi}_{s_0\sim\rho}\Sp{s_t=s}$ represents the probability that $s_t=s$ if the agent follows policy $\pi$ and the initial state $s_0$ is sampled from distribution $\rho$. 
We can easily verify that $\sum_{s\in\S}d^{\pi}_{\rho, s}=1$ and thus $d^\pi_\rho\in\R^{\abs{\S}}$ is a valid probability distribution. Meanwhile, by truncating all terms with $t\geq 1$ in the sum in~\eqref{equ:def_visitation}, we obtain $d^{\pi}_{\rho, s}\geq(1-\gamma)\rho_s$ for any $s\in\S$.

The gradient of $V^{\pi}_{\rho}$ with respect to $\pi$ is given by the policy gradient theorem \citep{sutton1999policy} as
\begin{equation}\label{eqn:policy-grad}
\textstyle
\nabla_s V^{\pi}_{\rho}:=\frac{\partial V^{\pi}_{\rho}}{\partial \pi_s}=\frac{1}{1-\gamma}d^{\pi}_{\rho, s}Q^{\pi}_{s}\in \R^{\abs{\A}}. 
\end{equation}
Then, we define $\nabla V^{\pi}_{\rho}\in\R^{\abs{\S}\times\abs{\A}}$ as the concatenation of $\nabla_s V^{\pi}_{\rho}$ for all $s\in\S$.


\subsection{Mirror Descent}
\label{sec:fmd_prelim}


Mirror descent (MD) is a general framework for the construction and analysis of optimization algorithms \citep{nemirovskij1983problem}.
Its key machinery is a pair of conjugate mirror maps that map the iterates of an optimization algorithm back-and-forth between a primal space and a dual space.  
We follow the common practice of defining the mirror maps with the gradient mapping of a convex function of \emph{Legendre-type} \citep[Section~26]{Rockafellar1970book}.


Let's first define \emph{Bregman divergence} and \emph{Bregman projection}.
Suppose that~$\Psi$ is a convex function of Legendre type.
It induces a Bregman divergence between any $x\in\dom\Phi$ and $y\in\intr(\dom\Phi)$:
\begin{equation}
    \label{equ:bregman}
    D_{\Phi}(x, y)=\Phi(x) - \Phi(y) - \inner{\nabla\Phi(y), x - y}.
\end{equation}
Let $\mc{C}\in\dom\Phi$ be a closed convex set. 
The \emph{Bregman projection} of any $y\in\intr(\dom\Phi)$ onto $\mc{C}$ is 
\begin{equation}\label{eqn:bregman-proj}
\proj^{\Phi}_{\mc{C}}(y) = \argmin_{x\in\mc{C}} D_{\Phi}(x, y).
\end{equation}
Properties of Bregman projection can be found in, e.g., \cite{Bauschke1997}.

\newpage

Now consider the problem of minimizing a convex function $f:\X\to\R\cup\{\infty\}$ over a closed convex set~$\mc{C}\subset\dom\Phi$.
We use the presentation of MD given by \citet{bubeck2015book}: at each iteration~$k$,
\begin{enumerate} 
\item 
Given $x^{(k)}$, find $y^{(k+1)}$ such that 
\begin{equation} \label{eqn:md-dual-update}
\nabla\Phi(y^{(k+1)}) = \nabla\Phi(x^{(k)}) - \eta_k g^{(k)}.
\end{equation}
where $\eta_k$ is the step size
and $g^{(k)}$ is the gradient $\nabla f(x^{(k)})$ or a sub-gradient of~$f$ at $x^{(k)}$.
\item Compute $x^{(k+1)} = \proj^{\Phi}_{\C}(y^{(k+1)})$.
\end{enumerate}

Define the \textit{conjugate function} of $\Phi$ as $\Phi^*(x^*)=\sup_{x\in\dom\Phi}\Bp{\inner{x, x^*}-\Phi(x)}$. Then, using the definition in~\eqref{eqn:bregman-proj} and the identity $\nabla\Phi^*(\nabla\Phi(x))=x$, we can express it more compactly as
\begin{equation}\label{eqn:md}
x^{(k+1)}\!\!=\argmin_{x\in\mc{C}}D_{\Phi}\!\left(\!x, \nabla\Phi^*\!\bigl(\nabla\Phi(x^{(k)})\!-\!\eta_k g^{(k)}\bigr)\!\right),
\end{equation}
which can be further simplified to \citep{beck2003mirror,bubeck2012regret}
\begin{equation}\label{eqn:md-prox} 
x^{(k+1)}\! =\argmin_{x\in\mc{C}}\Bp{\eta_k\bigl\langle g^{(k)}, x\bigr\rangle + D_{\Phi}(x, x^{(k)})}. 
\end{equation}


Next we discuss three examples of the MD algorithm for solving $\min_{x\in\Delta} f(x)$, where $\Delta$ is the simplex.
Each leads to a variant of the \algoacronym method we will present in Section~\ref{sec:compo}.
\vspace{1ex}
\begin{example}[Squared $L_2$-norm]\label{ex:squared-2norm}
Let $\Phi(x)=\frac{1}{2}\|x\|_2^2$, which is Legendre type with $\intr(\dom\Phi)=\dom\Phi=\R^n$. We have $\Phi^*(x^*)=\frac{1}{2}\|x^*\|_2^2$, $\nabla\Phi(x)=x$, $\nabla\Phi^*(x^*)=x^*$, and $\D_{\Phi}(x,y)=\frac{1}{2}\|x-y\|_2^2$.
In this case, the MD algorithm~\eqref{eqn:md} becomes the classical projected gradient method 
\begin{align*}
x^{(k+1)} 
&=\argmin_{x\in\Delta} \bigl\|x-\bigl(x^{(k)} - \eta_k \nabla f(x^{(k)})\bigr) \bigr\|_2^2.
\end{align*}
\end{example}

\begin{example}[Negative entropy on $\R^n_+$]\label{ex:ent-Rn+}
Consider the negative entropy 
$\Phi(x) = \sum_i (x_i\log(x_i) - x_i)$ 
with $\dom \Phi=\R^n_+$ (and the convention $0\log0 = 0$).
It is of Legendre type, with 
$\Phi^*(x^*) = \sum_i \exp(x^*_i)$, 
$\nabla\Phi(x)=\log(x)$ and
$\nabla\Phi^*(x^*)=\exp(x^*)$, 
where $\log$ and $\exp$ apply component-wise to vectors.
For any $x\in\R^n_+$ and $y\in\R^n_{++}$, 
their Bregman divergence is the KL-divergence:
\begin{equation}\label{eqn:KL-div}
\textstyle D_{\Phi}(x,y) = \sum_i (x_i\log(x_i/y_i) - x_i + y_i).
\end{equation}
In this case, the Bregman projection of $y\in\R^n_{++}$ onto~$\Delta$ is
$\proj_\Delta^{\Phi}(y) = y/\|y\|_1 $
and ~\eqref{eqn:md} becomes 
\begin{equation}\label{eqn:exp-grad}
x^{(k+1)} = x^{(k)}\exp(-\eta_k g^{(k)})/ \|x^{(k)}\exp(-\eta_k g^{(k)})\|_1.
\end{equation}
\end{example}

\begin{example}[Negative entropy on $\Delta$]\label{ex:ent-simplex}
Let $\phi(x)=\sum_i (x_i\log(x_i)-x_i)$ and define
$\Phi(x) = \phi(x) + \delta(x|\Delta)$, where $\delta(\cdot|\Delta)$ is the indicator function of~$\Delta$, i.e., $\delta(x|\Delta)=0$ if $x\in\Delta$ and $+\infty$ otherwise. 
Apparently $\dom\Phi=\Delta$, which has an empty interior.
As a result, $\Phi$ is not of Legendre type and in fact is not differentiable (see Appendix~\ref{sec:legendre}).
However, the MD algorithm is still well-defined. 
Specifically, in~\eqref{eqn:md-dual-update} we interpret $\nabla\Phi(x^{(k)})$ as any subgradient in the subdifferential $\partial\Phi(x^{(k)})$, and find $y^{(k+1)}$ such that there exists some $\nabla\Phi(y^{(k+1)})\in\partial\Phi(y^{(k+1)})$ to make the equality hold.

Despite $\nabla\Phi(y)$ being multi-valued as a subgradient, the Bregman divergence~\eqref{equ:bregman} is still well defined (Corollary \ref{coro:breg_well_defined}).
As a result, 
$D_\Phi$ is the same as~\eqref{eqn:KL-div}. Using the fact $x,y\in\Delta$, it can be simplified as
\begin{equation}\label{eqn:KL-simplex}
\textstyle D_{\Phi}(x,y) = \sum_i x_i\log(x_i/y_i).
\end{equation}

In addition, we have 
$\Phi^*(x^*) = \log\left(\sum_i\exp(x^*_i)\right)$
with $\dom\Phi^*=\R^n$
\citep[Section~16]{Rockafellar1970book}. 
Clearly, $\Phi^*$ is a differentiable function throughout $\R^n$ and
$$\nabla\Phi^*(x^*) = \exp(x^*)/\|\exp(x^*)\|_1.$$
In this case, the MD algorithm~\eqref{eqn:md} yields the same update as~\eqref{eqn:exp-grad}.
However, the projection step is no longer needed because the range of $\nabla\Phi^*$ is the interior of~$\Delta$ and we can simply express MD as
\[
x^{(k+1)}=\nabla\Phi^*\bigl(\nabla\Phi(x^{(k)})-\eta_k g^{(k)}\bigr).
\]
\end{example}

\begin{remark}
Although Examples~\ref{ex:ent-Rn+} and~\ref{ex:ent-simplex} give the same update~\eqref{eqn:exp-grad}, 
there are subtle differences in the theory.
In particular, we have $\nabla\Phi^*=(\nabla\Phi)^{-1}$ in Example~\ref{ex:ent-Rn+} and the equivalence between~\eqref{eqn:md} and~\ref{eqn:md-prox} is easily established.
However, this is not the case for Example~\ref{ex:ent-simplex} because~$\Phi$ is not Legendre and nondifferentiable.
Consequently, we can no longer leverage the convex optimization machinery as in the tabular case \cite[e.g.,][]{xiao2022convergence,lan2023policy} for convergence analysis with general function approximation.
Instead, we extend the classical MD theory to work with mirror maps whose gradient mapping and conjugate mapping are not inverses of each other; see Appendix~\ref{sec:legendre}, Lemma~\ref{lem:affine}.
\end{remark}


\section{Policy Optimization with Dual Function Approximation}
\label{sec:compo}



Recall the setting of MDP in Section~\ref{sec:mdp_prelim}.
In the tabular case, the policy mirror descent (PMD) method 
\citep{shani2020adaptive, lan2023policy, xiao2022convergence}
takes the form of~\eqref{eqn:md-prox}: 
\begin{equation}\label{eqn:pmd-tabular}
\pi^{\bkone}_{s}=\argmin_{\pi_s\in\Delta(\A)}\Bigl\{\eta_k\bigl\langle\widehat{Q}^{\bk}_{s}, \pi_s\bigr\rangle + D_{\Phi}(\pi_s, \pi^{\bk}_s)\Bigr\},
\quad s\in\S,
\end{equation}
where $\widehat{Q}^{\bk}$ is some approximation of $Q^{(k)}$.
We note that $Q^{(k)}$ is not the gradient of the value function $V_\rho^\pi$ at $\pi^{(k)}$, which is given in~\eqref{eqn:policy-grad}; rather, it is a preconditioned gradient \citep{kakade2001natural}. 


When the size of the state-action space becomes large (possibly infinite), we have to resort to function approximation.
Specifically, let $\pi^\theta$
be a differentiable mapping from the set of parameters $\Theta\subset\R^n$ to the set of stochastic policies.
The parameter update step corresponding to~\eqref{eqn:pmd-tabular} becomes
\begin{equation}\label{eqn:MDPO-theta-update}
\theta^{(k+1)} \!= \argmin_{\theta\in\Theta}\; 
\E_{s\sim d^{(k)}_\rho}\!\!\left[ \E_{a\sim \pi^{\theta}_s}\!\bigl[\widehat{Q}^{(k)}_{s,a}\bigr] \!+\! D_{\Phi}(\pi^{\theta}_s,\pi^{(k)}_s)\right],
\end{equation}
where $\pi^{(k)}$ means $\pi^{\theta^{(k)}}$, and $d^{(k)}_\rho$ and $\widehat{Q}^{(k)}_{s,a}$ are simple notations for $d^{\pi^{(k)}}_\rho$ and $\widehat{Q}^{\pi^{(k)}}_{s,a}$ respectively.
This approach is adopted by, e.g., \citet{tomar2020mirror} and \citet{vaswani2021general}.
However, the optimization problem is no longer convex in~$\theta$, and its convergence analysis becomes more challenging.

\citet{alfano2023novel} introduced Approximate Mirror Policy Optimization (AMPO), a framework that incorporates general parametrization into PMD with convergence guarantees.
A key instrument they introduced is the \emph{Bregman projected policy} class. 
The idea is to use a parametrized function $f^\theta:\S\times\A\to\R$ to approximate the dual update in~\eqref{eqn:md-dual-update}, which in the context of PMD is
$\nabla\Phi(\pi^{(k)})-\eta_k \widehat{Q}^{(k)}$.
Then follow the second step in MD to define the policy class
\[
\big\{\pi^\theta: \pi^\theta_s = \proj^\Phi_{\Delta(\A)}\bigl(\nabla\Phi^*(f^\theta_s)\bigr),~s\in\S\bigr\}, \quad \theta\in\Theta.
\]
For example, using the negative-entropy 
(Example~\ref{ex:ent-Rn+} or~\ref{ex:ent-simplex}), it leads to the softmax policy class: 
\begin{equation}\label{eqn:soft-max-proj}
\pi_{s,a}^\theta = \exp(f_{s,a}^\theta)/\|\exp(f_s^\theta)\|_1,
\quad (s,a)\in\S\times\A.
\end{equation}
While such policy classes are widely used in both theory and practice, recognizing them as the composition of a Bregman projection, a conjugate mirror map and a generic function approximation~$f^\theta$ (such as neural networks) allows more structured and sharper convergence analysis.

\begin{algorithm}[t]
\setstretch{1.1}
\caption{\algofullname (\algoacronym)}
\label{alg:compo}
\begin{algorithmic}[1]
    \STATE \textbf{Input:} Initialize policy $\pi^{(0)}$ with parameters $\theta^{(0)}$; mirror map $\Phi$
        \FOR{$k=0, \dots, K-1$}
            \STATE Find $\widehat{Q}^{\bk}$ that approximates $Q^{\bk}$ (Critic Update)
            \STATE 
            Find $\theta^{\bkone}$ that (approximately) solves the problem
$$\min_{\theta\in\Theta}\; 
\E_{s\sim d^{(k)}_{\rho}}\Bigl[D_{\Phi^*}\bigl(\nabla\Phi(\pi^{(k)}_s) - \eta_k \widehat{Q}^{\bk}_s, \, f^{\theta}_s\bigr)\Bigr]$$
\label{line:find_theta}
            \STATE 
            Assign $\pi^{\bkone}_s = \proj_{\Delta(\mc{A})}^{\Phi}\bigl(\nabla\Phi^*(f_s^{\theta^{\bkone}})\bigr)$, ~$s\in\mc{S}$
            \label{line:projection}
        \ENDFOR
    \end{algorithmic}
\end{algorithm}

Facilitated with the Bregman projected policy class,
extending PMD with function approximation rests upon how we approximate $\nabla\Phi(\pi^{(k)})-\eta_k\widehat{Q}^{(k)}$ 
(\emph{existing in the dual space}) using $f^\theta$.
AMPO \citep{alfano2023novel} proposes to minimize
the expected $L_2$-distance between them, i.e., 
\begin{equation}\label{equ:loss_apmo}
    \min_{\theta}\;\E_{s\sim d^{(k)}_\rho}\Bigl[\bigl\|f^{\theta}_s - \bigl(\nabla\Phi(\pi^{\bk}_s) - \eta_k\widehat{Q}_s^{\bk}\bigr)\bigr\|_2^2\Bigr].
\end{equation}
On the other hand, \citet{lan2022policy} tries to minimize the expected (in state distribution) $L_{\infty}$-norm of the difference between 
$f^{\theta}_s$ and $\nabla\Phi(\pi^{\bk}_s)-\eta\widehat{Q}^{\bk}_s$.

\revision{In contrast, we propose to use the corresponding \textit{dual Bregman divergence} $D_{\Phi^*}$ to measure their similarity in the dual space.
In particular, Our method finds $\theta^{(k+1)}$ by (approximately) solving
\begin{equation}\label{eqn:compo-update}
\min_{\theta\in\Theta}\; 
\E_{s\sim d^{(k)}_{\rho}}\Bigl[D_{\Phi^*}\bigl(\nabla\Phi(\pi^{(k)}_s) - \eta_k \widehat{Q}^{\bk}_s, \, f^{\theta}_s\bigr)\Bigr],
\end{equation}
where $d^{(k)}_\rho$ can be replaced with other distributions to accommodate the scenario of off-policy training. Here, we can see that the similarity between the two dual vectors
$f^{\theta}_s$ and $\nabla\Phi(\pi^{\bk}_s)-\eta\widehat{Q}^{\bk}_s$
are measured by the Bregman divergence of $\Phi^*$, which naturally lives in the dual space. 
Together with the Bregman divergence of~$\Phi$ used in policy projection, they form a complete duality framework.
}
A complete description of our method is given as Algorithm~\ref{alg:compo}, 
and we call it \algofullname (\algoacronym). 

\subsection{Instantiations of \algoacronym}
We give three instantiations of \algoacronym using the three mirror maps given in Examples~\ref{ex:squared-2norm}-\ref{ex:ent-simplex}. 
\revision{In deriving these instantiations as well as implementing the algorithms, instead of directly using the dual Bregman divergence $D_{\Phi^*}$, it is often more convenient to use the following identity:
\begin{equation}
    \label{equ:bregman_duality}
    D_{\Phi^*}\bigl(\nabla\Phi(\pi^{(k)}_s) - \eta_k \widehat{Q}^{\bk}_s, \, f^{\theta}_s\bigr)=D_{\Phi}\bigl(\nabla\Phi^*(f^{\theta}_s),\, \nabla\Phi^*(\nabla\Phi(\pi^{(k)}_s) - \eta_k \widehat{Q}^{\bk}_s)\bigr).
\end{equation}
See Corollary \ref{coro:dual_bregman} for a proof.
This identity will also facilitate our convergence analysis later.}

\textbf{\algoacronym-$L_2$.}
With $\Phi$ being the squared $L_2$-norm mirror map described in Example~\ref{ex:squared-2norm}, the approximation problem in~\eqref{eqn:compo-update} (same as line~\ref{line:find_theta} in Algorithm~\ref{alg:compo}) becomes
\begin{equation}\label{eqn:compo-L2-loss}
\min_{\theta}\; \E_{s\sim d^{(k)}_\rho}\left[ \bigl\|f_s^{\theta} - \pi^{(k)}_s + \eta_k\widehat{Q}_s^{(k)}\bigr\|_2^2\right], 
\end{equation}
and Line~\ref{line:projection} of Algorithm~\ref{alg:compo} is the Euclidean projection
\[
\pi^{(k+1)}_s = \argmin_{\pi\in\Delta(\mc{A})}\; \bigl\|\pi - f_s^{(k+1)}\bigr\|_2^2 .
\]
Here we have used the simpler notation $f_s^{(k+1)}$ for $f_s^{\theta^{(k+1)}}$.

\textbf{\algoacronym-KL$^*$.}
With $\Phi$ being the negative entropy defined on $\R^{|\A|}_+$ (see Example~\ref{ex:ent-Rn+}), we have
\begin{align*}
\nabla\Phi^*(f_s^{(k+1)}) &= \exp(f_{s,a}^\theta), \\
\nabla\Phi^*\bigl(\nabla\Phi(\pi^{(k)})-\eta_k\widehat{Q}^{(k)}\bigr) &= \pi^{(k)}_s \exp\bigl(-\eta_k\widehat{Q}^{(k)}_s\bigr) .
\end{align*}
Using the identity~\eqref{equ:bregman_duality}, we can write the loss in the approximation problem~\eqref{eqn:compo-update} as
\begin{equation}
    \label{eqn:compo-KL*-loss}
    \E_{s\sim d^{(k)}_\rho} \! \Mp{D_{\KL}\Sp{\exp(f^\theta_s) \big\| \pi_s^{\bk}\exp\bigl(-\eta_k\widehat{Q}^{\bk}_s\bigr)}},
\end{equation}
where $D_{\KL}$ is given by~\eqref{eqn:KL-div}.
Policy projection as in~\eqref{eqn:soft-max-proj} is necessary to obtain $\pi^{(k+1)}$ because $\nabla\Phi^*(f_s^{(k+1)})$ is not in the simplex in general.
\algoacronym-KL$^*$ has disadvantages in both theory and practice compared with its close variant \algoacronym-KL, which we will explain next.

\textbf{\algoacronym-KL.} 
With $\Phi$ being the negative entropy restricted on $\Delta(\A)$ (see Example~\ref{ex:ent-simplex}), the range of $\nabla\Phi^*$ is $\Delta(\A)$, thus the projection step (Line~\ref{line:projection} of Algorithm~\ref{alg:compo}) becomes redundant. Here, we have
\begin{align*}
\pi^{(k+1)}_s = \nabla\Phi^*(f_s^{(k+1)}) &= \exp(f_{s,a}^\theta)/\|\exp(f_s^\theta)\|_1 , \\
\nabla\Phi^*\bigl(\nabla\Phi(\pi^{(k)})-\eta_k\widehat{Q}^{(k)}\bigr) &= \pi^{(k)}_s \exp\bigl(-\eta_k\widehat{Q}^{(k)}_s\bigr) / Z^{(k)}_s,
\end{align*}
where 
$Z^{(k)}_s= \bigl\|\pi^{(k)}_s \exp\bigl(-\eta_k\widehat{Q}^{(k)}_s\bigr)\bigr\|_1$.
Again using~\eqref{equ:bregman_duality}, the loss in~\eqref{eqn:compo-update} becomes
\begin{equation}
    \label{eqn:compo-KL-loss}
    \E_{s\sim d^{(k)}_\rho} \! \Mp{D_{\KL}\Sp{\pi^{\theta}_s\ \Big\|\, \pi_s^{\bk}\exp\bigl(-\eta_k\widehat{Q}^{\bk}_s\bigr)/Z^{\bk}_s}},
\end{equation}
where $D_{\KL}$ is given by~\eqref{eqn:KL-simplex}.
There are several distinctions between \algoacronym-KL and
\algoacronym-KL$^*$.
\begin{itemize} 
    \item 
    The approximation loss in~\eqref{eqn:compo-KL-loss} is in terms of the full policy parametrization $\pi^\theta$ (normalized over the simplex), matching the implementation of several popular algorithms \citep{tomar2020mirror,vaswani2021general}.
    In contrast, the loss in~\eqref{eqn:compo-KL*-loss} is in terms of the unnormalized entity $\exp(f^\theta)$, which will suffer additional loss after policy projection. 
    \item In theory,we are able to provide a competitive convergence analysis of \algoacronym-KL (see Section~\ref{sec:analysis-KL}) thanks to the fact that the two arguments in $D_{\KL}$ in~\eqref{eqn:compo-KL-loss} are both on the simplex, which is not the case in~\eqref{eqn:compo-KL*-loss}.
\end{itemize}
For these reasons, we will only consider \algoacronym-KL from now on. However, we think it is necessary to expose the subtleties between the two variants, because many works on policy mirror descent methods 
\cite[e.g.][]{alfano2023novel}, 
assumes $\nabla\Phi^*=(\nabla\Phi)^{-1}$.
We demonstrate that the more nuanced extension of the MD theory (Lemma~\ref{lem:affine}) is crucial for developing and analyzing practical algorithms.

\subsection{Comparison with AMPO, MDPO and FMA-PG}
\label{sec:compare_ampo}

AMPO \citep{alfano2023novel} replaces the minimization problem in Line~\ref{line:find_theta} of Algorithm~\ref{alg:compo} by~\eqref{equ:loss_apmo} regardless of the mirror map used in policy projection.
%
More concretely, let $\Phi_1$ be the negative entropy on $\R^n_+$ and $\Phi_2$ be the squared $L_2$ norm. 
Then AMPO's approximation loss can be written as 
\begin{equation}\label{eqn:ampo-phi-1-2}
\E_{s\sim d^{(k)}_\rho}\! \Mp{D_{\Phi_2^*}\!\Sp{\nabla\Phi_1(\pi^{\bk}_s) - \eta_k\widehat{Q}_s^{\bk}, f^{\theta}_s}},
\end{equation}
and $\Phi_1$ is again used in the policy projection step. 
In theory, as long as the approximation error is small, it is possible to establish convergence of the method \citep{alfano2023novel}.
However, such a mismatch, or \revision{inconsistency, between approximations in primal and dual spaces} may cause problems when the approximation error cannot be made sufficiently small. 
This is precisely the case in practice, where we can only afford to run at most a few steps of the stochastic gradient method to reduce the approximation error.
The importance of the consistency between the two mirror maps has also been pointed out by \citet{tomar2020mirror}.

In Section~\ref{sec:experiments}, we demonstrate that on standard benchmarks \algoacronym-KL obtains state-of-the-art performance with only one step of stochastic gradient method in reducing the approximation loss~\eqref{eqn:compo-update}, comparable to SAC \citep{haarnoja2018soft}.
On the other hand, we could not get AMPO competitive with many numbers of stochastic gradient steps.


The Mirror Descent Policy Optimization (MDPO) method of \citet{tomar2020mirror} is based on minimizing over~$\theta$ directly in the formulation~\eqref{eqn:MDPO-theta-update}. If $\pi^\theta$ belongs to the softmax class of~\eqref{eqn:soft-max-proj} and $D_{\Phi}$ is the KL-divergence, then it is equivalent to \algoacronym-KL.
Therefore, our convergence analysis in Section~\ref{sec:analysis-KL} directly applies to MDPO, which is not provided by \citet{tomar2020mirror}.

The Functional Mirror Ascent (FMA-PG) framework of~\citet{vaswani2021general} also takes the form~\eqref{eqn:MDPO-theta-update}. However, similar to MDPO, \citet{vaswani2021general} did not exploit any composition structure of the parametrization $\pi^\theta$ or the MDP structure. Rather, they conducted convergence analysis based on the general theory for smooth, non-convex optimization, which leads to considerably weaker results. 


\revision{
\subsection{SAC as a special case of \algoacronym-KL}}
\label{sec:equivalence}

Soft Actor-Critic (SAC) \citep{haarnoja2018soft2} is a very popular reinforcement learning algorithm, which was developed under the framework of entropy-regularized reinforcement learning. \citet{tomar2020mirror} compared SAC's actor update loss function with~\eqref{eqn:compo-KL-loss} and pointed out that SAC is similar to MDPO 
(same as \algoacronym-KL, as we discussed above) 
by replacing the previous iterate $\pi^{(k)}$ with the uniform distribution. Here, we will prove a much stronger result, showing that by choosing the learning rate $\eta_k$ appropriately, Equation~\eqref{eqn:compo-KL-loss} become exactly SAC's policy update rule.

To prove this, we will first briefly introduce the framework of entropy-regularized reinforcement learning and then derive the corresponding \algoacronym-KL algorithm under this framework. 
\footnote{See \citet{cen2022fast} and \citet{cayci2021linear} for backgrounds on entropy-regularized reinforcement learning.} 
With a regularization parameter $\tau>0$, the regularized value function in this framework is
\begin{equation*}
    \textstyle
     V^\pi_{\tau, \rho} = \E_{a_t\sim\pi_{s_t}}\bigl[\sum_{t=0}^{\infty}\gamma^t\Sp{ c(s_t, a_t)+\tau\log\pi(a_t\mid s_t)} ~|~  s_0\sim\rho\bigr].
\end{equation*}
%
Then, we can similarly define the Q-value function as
\begin{equation}\label{equ:def_Q_entropy}
    \textstyle
    Q^{\pi}_{\tau}(s, a)=\E_{a_t\sim \pi_{s_t}}\bigl[\sum_{t=0}^{\infty}\gamma^t \Sp{ c(s_t, a_t)+\tau\log\pi(a_t\mid s_t)} ~|~ s_0=s, a_0=a\bigr].
\end{equation}
As shown in \citet{cayci2021linear}, the policy gradient for this entropy-regularized value function is
\begin{equation*}
    \textstyle
    \nabla_s V^{\pi}_{\tau, \rho}=\frac{1}{1-\gamma}d^{\pi}_{\rho, s}Q^{\pi}_{\tau, s}\in\R^{\abs{\A}}.
\end{equation*}
Therefore, we can obtain the corresponding \algoacronym algorithm by using this policy gradient. 
Setting~$\Phi$ as the negative entropy, we obtain the corresponding \algoacronym-KL update rule as
\begin{equation}
    \label{equ:sac_like_loss_regularize}
    \theta^{\bkone}\in\argmin_{\theta}\E_{s\sim d^{\bk}_{\rho}}\Mp{D_{\KL}\Sp{\pi^{\theta}_{s}\ \left\|\  \pi^{\bk}_{s}\exp\Sp{-\eta_k Q^{\bk}_{\tau, s}} / Z^{\bk}_{s}\right.}},
\end{equation}
where $Z^{\bk}_{s}$ is the normalization factor and we take the exact Q-value function for simplicity.

Now, we switch our attention to the SAC algorithm in \citet{haarnoja2018soft2}. 
The subtlety is that the soft Q-value function used in \citet{haarnoja2018soft2} is defined differently from the one in equation~\eqref{equ:def_Q_entropy}.
To be more clear, we denote it as $q^{\pi}_{\tau}$, which for any $(s, a)\in\S\times\A$ is defined as
\begin{equation}
    \label{equ:def_Q_entropy_alter}
    \begin{cases}
        q^{\pi}_{\tau}(s, a) = c(s, a) + \gamma\E_{s'\sim\mc{P}(s, a)}\Mp{V^{\pi}_{\tau}(s')},\\
        V^{\pi}_{\tau}(s) = \E_{a'\sim\pi_{s}}\Mp{\tau\log\pi(a'\mid s) + q^\pi_{\tau}(s, a')}.
    \end{cases}
\end{equation}
Note that the definition of $V^{\pi}_{\tau}$ remains unaffected. Then, we can immediately obtain the relation $q^{\pi}_{\tau}(s, a)=Q^{\pi}_{\tau}(s, a) -\tau\log\pi(a \mid s)$.
As a result, 
the policy update rule in SAC is\footnote{The original proposition of SAC in \citet{haarnoja2018soft2} uses $\exp(q^{\bk}_{\tau, s}/\tau)$ instead of $\exp(-q^{\bk}_{\tau, s}/\tau)$ because it considers reward maximization instead of cost minimization.}
\begin{align*}
    \theta^{\bkone}\in& \argmin_{\theta}\E_{s\sim d^{\bk}_{\rho}}\Mp{D_{\KL}\Sp{\pi^{\theta}_{s}\ \middle\|\  \exp\Sp{-\nicefrac{q^{\bk}_{\tau, s}}{\tau}} / Z^{\bk}_{s}}}\\
    =&\argmin_{\theta}\E_{s\sim d^{\bk}_{\rho}}\Mp{D_{\KL}\Sp{\pi^{\theta}_{s}\ \middle\|\  \exp\Sp{-\nicefrac{Q^{\bk}_{\tau, s}}{\tau} + \log\pi^{\bk}_{s}} / Z^{\bk}_{s}}}\\
    =&\argmin_{\theta}\E_{s\sim d^{\bk}_{\rho}}\Mp{D_{\KL}\Sp{\pi^{\theta}_{s}\ \middle\|\  \pi^{\bk}_{s}\exp\Sp{-\nicefrac{Q^{\bk}_{\tau, s}}{\tau}} / Z^{\bk}_{s}}},
\end{align*}
which is exactly the same as the update rule in Eq. \eqref{equ:sac_like_loss_regularize} if we take $\eta_k=\frac{1}{\tau}$ for any $k$. Therefore, we conclude that \textit{SAC's update rule can be obtained by taking $\eta_k=\frac{1}{\tau}$ for any $k$ in \algoacronym-KL for an entropy-regularized MDP.} As an immediate consequence, we can have a tight convergence rate analysis for SAC, as given in Section~\ref{sec:analysis-sac}.
\bigskip

\section{Convergence Analysis}
\label{sec:convergence}
In this section, we present the convergence analysis of \algoacronym-KL and SAC. 
These results are nontrivial extensions of similar results for PMD method in the tabular case \citep{xiao2022convergence} and with the log-linear policy class \citep{yuan2022linear}.
The analysis of \algoacronym-$L_2$ is deferred to Appendix~\ref{sec:analysis-L2}.

\subsection{Analysis of \algoacronym-KL}
\label{sec:analysis-KL}

We make following three assumptions about running Algorithm~\ref{alg:compo}, with initial distribution $\rho\in\Delta(\S)$.
\vspace{-1em}
 \begin{enumerate}\itemsep 0pt
    \item[(\textsf{A1})] 
    There exist constants $\epsilon_{\critic}, \epsilon_{\actor}>0$ such that for every iteration~$k$, it holds that
        \begin{equation}
            \label{equ:error_assumption}
            \begin{split}
                &\E_{s\sim d^{(k)}_\rho} \!
                \Bigl[\bigl\|\widehat{Q}^{\bk}_s-Q^{\bk}_s\bigr\|_{\infty}\Bigr]\leq\epsilon_{\critic},\\
                & \E_{s\sim d^{(k)}_\rho} \! \Mp{D_{\KL}\Sp{\pi^{(k+1)}_s \Big\|\, \pi_s^{\bk}\exp\bigl(-\eta_k\widehat{Q}^{\bk}_s\bigr) / Z^{\bk}_s}} \leq \eta_k\epsilon_{\actor}
            \end{split}
        \end{equation}

    \item[(\textsf{A2})] 
    There exists a constant $\vartheta_{\rho}\geq 1$ such that for any $k$, 
    $$\max\Bp{\Norm{\frac{d^{\star}_{\rho}}{d^{\bkone}_{\rho}}}_{\infty}, \Norm{\frac{d^{\bkone}_{\rho}}{d^{\bk}_{\rho}}}_{\infty}, \Norm{\frac{d^{\bkone}_{d^\star_\rho}}{d^{\bk}_{\rho}}}_{\infty}, \Norm{\frac{d^{\bkone}_{d^\star_\rho}}{d^{\star}_{\rho}}}_{\infty}}\leq\vartheta_{\rho}.$$

    \item[(\textsf{A3})] There exists a constant $C_{\rho}>0$ such that for any $k$,
    $$\max_{s\in\mathrm{supp}\bigl(d^{\bk}_\rho\bigr)}\Bp{\Norm{\frac{\pi^\star_s}{\pi^{\bkone}_s}}_{\infty}, ~\Norm{\frac{\pi^{\bk}_s}{\pi^{\bkone}_s}}_{\infty}}\leq C_{\rho}.$$
\end{enumerate}
\vspace{-0.5em}

Here, in (\textsf{A1}), we assume that $\widehat{Q}^{\bk}$ is a good enough approximation of $Q^{\bk}$, which is a problem that has been extensively studied both theoretically and empirically \citep{li2023policy, chen2022target, fujimoto2018addressing}. 
We also assume that the parameterized function $f^{\theta}$ is powerful enough to approximate the dual vector $\nabla \Phi(\pi^{\bk})-\eta_k\widehat{Q}^{\bk}$, which is a common assumption for studying function approximation \citep{alfano2023novel, lan2022policy, agarwal2021theory}. 

Then, (\textsf{A2}) assumes that the distribution mismatch coefficient is bounded, which is often needed for analyzing policy gradient methods \citep{xiao2022convergence, yuan2022linear} and can be satisfied if we take $\rho=\textsf{Unif}(\S)$ (see Lemma \ref{lem:dis_mismatch_bound}). Meanwhile, (\textsf{A3}) is an assumption on policy evolution. It holds, for example, when we apply \algoacronym-KL with entropy regularization
\citep{cayci2021linear, cen2022fast}
and it covers the case of SAC (see discussion in Section~\ref{sec:equivalence}).

Under these assumptions, we have the following theorem and the proof is given in Appendix \ref{sec:proof_general}.
\begin{restatable}[Linear Convergence of \algoacronym-KL]{theorem}{convkl}
    \label{theo:convergence_kl_bapmd}
    Consider Algorithm \ref{alg:compo} with initial policy $\pi^{(0)}$, initial distribution $\rho\in\Delta(\S)$ and $\Phi$ being the negative entropy restricted on $\Delta(\A)$. Suppose Assumptions (\textsf{A1}), (\textsf{A2}) and (\textsf{A3}) hold and the step sizes satisfy $\eta_0>1$ and 
$\eta_{k+1}\geq\left(\vartheta_{\rho}/(\vartheta_{\rho}-1)\right)\eta_k$ for all $k\geq 0$. 
    Then, for any comparator policy $\pi^{\star}$, 
    we have
    \IfTwoColumnElse{
        \begin{align*}
            &V^{(K)}_{\rho}-V^\star_\rho\leq\Sp{1-\frac{1}{\vartheta_{\rho}}}^{\!\!K}\!\!\left(V^{(0)}_{\rho}-V^\star_\rho + \frac{D^\star_0 / (\vartheta_{\rho}-1)}{(1-\gamma)\eta_0} \right)\\
            &\quad  +\frac{\vartheta_{\rho}^2(1+C_{\rho})\Sp{\epsilon_{\actor} + \sqrt{2\epsilon_{\actor}}} + 2\vartheta_{\rho}\epsilon_{\critic}}{1-\gamma}.
        \end{align*}
    }{
        $$V^{(K)}_{\rho}-V^\star_\rho\leq\Sp{1-\frac{1}{\vartheta_{\rho}}}^K\left(V^{(0)}_{\rho}-V^\star_\rho + \frac{D^\star_0 / (\vartheta_{\rho}-1)}{(1-\gamma)\eta_0} \right)  +\frac{\vartheta_{\rho}^2\psi(\epsilon_{\actor}) + 2\vartheta_{\rho}\epsilon_{\critic}}{1-\gamma},$$
    }
    where $\psi(x)=(1+C_\rho)\Sp{x+\sqrt{2x}}$ for $x\geq 0$.
\end{restatable}

\begin{remark}
Theorem~\ref{theo:convergence_kl_bapmd} obtains linear convergence (up to an error floor dictated by $\epsilon_{\actor}$ and $\epsilon_{\critic}$) by employing a geometrically growing learning rate.  
We can also obtain $O(1/K)$ sublinear rates with a constant step size following similar proof techniques.
We omit the details as such results can be found in, for example, \citet{xiao2022convergence,yuan2022linear,alfano2023novel}.
\end{remark}

\revision{
\subsection{Analysis of SAC}
\label{sec:analysis-sac}
}
Although we have shown that SAC is a special case of \algoacronym-KL in Section \ref{sec:equivalence}, its convergence analysis is somewhat different from the above, as it is formulated under the framework of entropy-regularized reinforcement learning. Specifically, the key difference in analysis lies in the following modified performance difference lemma.
\begin{restatable}[Modified Performance Difference Lemma]{lemma}{mperdiff}
    \label{lem:sac_performance_diff}
	For any two policies $\pi, \tilde{\pi}:\S\mapsto\Delta(\A)$, initial distribution $\rho\in\Delta(\S)$ and regularization strength $\tau>0$, it holds that
        \begin{align*}
            V_{\tau, \rho}^{\pi} - V_{\tau, \rho}^{\tilde{\pi}} &= \frac{1}{1-\gamma}\E_{s\sim d^{\pi}_{\rho}}\Mp{\inner{Q^{\tilde{\pi}}_{\tau, s}, \pi_s - \tilde{\pi}_s} + \tau D_{\KL}(\pi_{s}\|\tilde{\pi}_{s})}\\
            &= \frac{1}{1-\gamma}\E_{s\sim d^{\tilde{\pi}}_{\rho}}\Mp{\inner{Q^{\pi}_{\tau, s}, \pi_s - \tilde{\pi}_s} - \tau D_{\KL}(\tilde{\pi}_s \| \pi_{s})}.
        \end{align*}
\end{restatable}

The proof is given in Appendix \ref{sec:proof_sac}. Using Lemma \ref{lem:sac_performance_diff}, the convergence guarantee of \algoacronym-KL with entropy regularization (and thus SAC) is summarized in the following theorem.

\begin{restatable}[Sublinear Convergence of SAC]{theorem}{convsac}
    \label{theo:convergence_sac}
    Consider running Algorithm \ref{alg:compo} for entropy-regularized reinforcement learning with initial policy $\pi^{(0)}$, regularization strength $\tau$, initial distribution $\rho\in\Delta(\S)$ and $\Phi$ being the negative entropy restricted on $\Delta(\A)$. Suppose Assumptions (\textsf{A1}), (\textsf{A2}) and (\textsf{A3}) hold and the step sizes satisfy $\eta_k=\eta\leq \frac{1}{\tau\vartheta_\rho}$ for any $k$. 
    Then, for any comparator policy $\pi^{\star}$, we have
    $$\frac{1}{K}\sum_{k=0}^{K-1}\Sp{V^{\bk}_{\tau, \rho} - V^{\star}_{\tau, \rho}}
    ~\leq~ \frac{1}{K}\Biggl(\frac{D^\star_0}{(1-\gamma)\eta}+\frac{V^{(0)}_{\tau, d^\star_\rho}}{1-\gamma}\Biggr) + \frac{\vartheta_{\rho}\psi(\epsilon_{\actor})+ (2-\gamma)\vartheta_\rho\epsilon_{\critic}}{(1-\gamma)^2},$$
    where $D^{\star}_0=\E_{s\sim d^\star_{\rho}}\bigl[D_{\KL}\bigl(\pi^\star_s ~\|~ \pi^{(0)}_s\bigr)\bigr]$ and $\psi(x)=(1+C_\rho)\Sp{x+\sqrt{2x}}$ for $x\geq 0$.
\end{restatable}

The full proof is given in Appendix \ref{sec:proof_sac}. To the best of our knowledge, this is the first convergence rate analysis of the SAC algorithm under general function approximation.
\section{Experiments}
\label{sec:experiments}

\begin{figure*}[t]
    \centering
    \includegraphics[width=\linewidth]{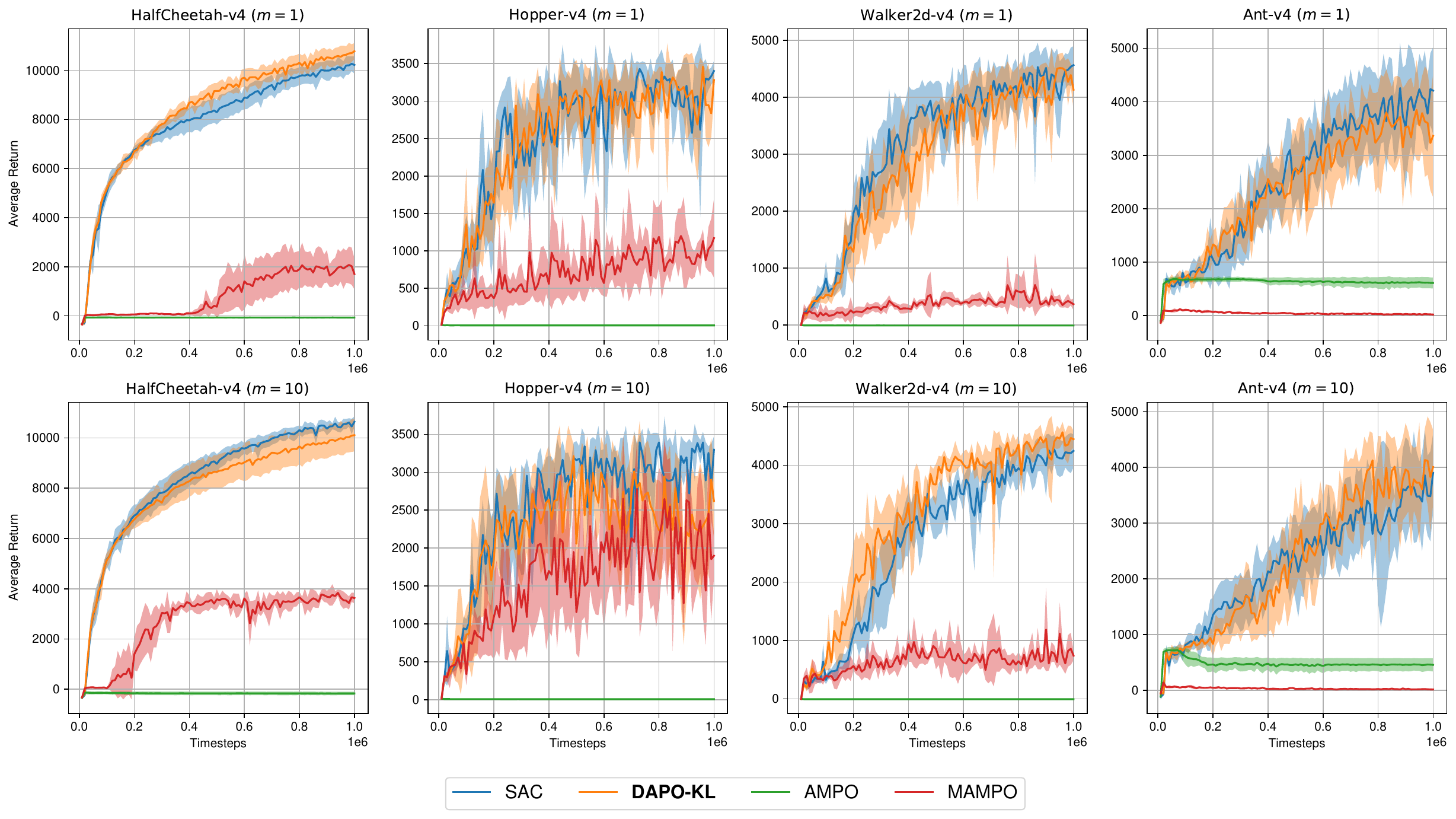}
    \caption{Average return curves on MuJoCo benchmarks. Each curve is averaged over 5 random seeds and the shaded area represents the $95\%$ confidence interval. Here~$m$ represents the number of stochastic gradient steps in each policy update iteration.}
    \label{fig:exp_results}
\vspace{-1ex}
\end{figure*}

In this section, we present our experiment results on several standard MuJoCo benchmark tasks \citep{todorov2012mujoco}. 
We compare the performance of \algoacronym-KL, SAC \citep{haarnoja2018soft}, and AMPO \citep{alfano2023novel}. 

In order to demonstrate the importance of \revision{having primal-dual consistency}, we modify AMPO to enforce it albeit in a naive way.
Specifically, as discussed in Section \ref{sec:compare_ampo}, AMPO's approximation loss can be expressed as~\eqref{eqn:ampo-phi-1-2} 
with $\Phi_1$ being negative entropy and $\Phi_2$ being  squared $L_2$-norm. Therefore, a naive way to enforce \revision{primal-dual consistency} is to replace $\Phi_1$ by $\Phi_2$ in~\eqref{eqn:ampo-phi-1-2}, which gives
\vspace{-1ex}
\begin{equation}
    \label{equ:loss_modified_ampo}
    \E_{s\sim d^{\bk}_{\rho}}\Bigl[\bigl\|f^{\theta}_s - \bigl(\pi^{\bk}_s - \eta_k\widehat{Q}_s^{\bk}\bigr)\bigr\|_2^2\Bigr].
\end{equation}
We call this algorithm Modified AMPO (MAMPO). Note that in MAMPO, $\Phi_1$ (negative entropy) is still used in the policy projection step, which is different from \algoacronym-$L_2$.

Although in theory we assume that the policy optimization loss is approximately minimized in each iteration, in practice, it may be only feasible to run a few steps of the stochastic gradient method to reduce the loss. 
Therefore, the number of stochastic gradient steps per iteration can be an important hyper-parameter for the algorithm. In experiments, all algorithms are evaluated under both $m=1$ and $m=10$ stochastic gradient step per iteration. Implementation details are given in Appendix \ref{sec:implementation}.

The results are summarized in Fig. \ref{fig:exp_results}. From the plots, we can see that \algoacronym-KL performs about the same as SAC on all tasks, which is expected as we have shown that SAC is a special case of \algoacronym-KL. Meanwhile, they are not sensitive to number of stochastic  gradient steps per iteration. 

On the other hand, AMPO fails to learn anything non-trivial on all tasks no matter it uses $m=1$ or $m=10$ stochastic gradient steps.
Nevertheless, we retain the possibility that our implementation of AMPO may not be the optimal and provide more details of its hyperparameter tuning in Appendix \ref{sec:additional_results}. 
In contrast, MAMPO is able to complete non-trivial learning among three tasks and gets better with more gradient steps, indicating the benefit of \revision{the primal-dual consistency} in~\eqref{equ:loss_modified_ampo}.
However, it is still far inferior to \algoacronym-KL and SAC.
\section{Conclusions}
\label{sec:conclusions}

\algoacronym is a novel duality framework for incorporating general function approximation into policy mirror descent methods.
Besides the mirror map in policy projection, 
it uses the dual mirror map for measuring the function approximation error.
We establish linear and sublinear convergence rates of \algoacronym under different step size rules and show that 
it incorporates state-of-the-art algorithms like SAC as a special case, immediately providing them with strong convergence guarantees. 



For future directions, \algoacronym paves the way for exploring new variants of PMD methods based on different mirror maps, e.g., with the negative Tsallis entropy. 
Another interesting question to investigate is how to characterize the effects of using inconsistent mirror maps in AMPO.

\bibliography{References}
\bibliographystyle{plainnat}


\newpage
\appendix

\section{Related Work}
\label{sec:detailed_related_work}

\paragraph{PG and PMD in tabular MDPs.} Although the proposal of policy gradient theorem and natural policy gradient (NPG) can be traced back to around 2000s or even before \citep{williams1992simple, konda1999actor, sutton1999policy, kakade2001natural}, the study of its convergence to the global optimum only started in recent years. On the other hand, mirror descent algorithm \citep{nemirovskij1983problem} has been extensively studied for a long time as an online learning algorithm \citet{bubeck2012regret}. To connect these two, \citet{neu2017unified} first shows that NPG can be viewed as a special case of policy mirror descent (PMD) and most of the following convergence analyses are based on this viewpoint. For tabular MDPs, \citet{shani2020adaptive} shows that unregularized NPG with a softmax policy has a $O(1/\sqrt{K})$ convergence rate. \citet{agarwal2021theory, vieillard2020leverage, xu2020improving} then improve it to the $O(1/K)$ convergence rate under different settings. After that, \citet{khodadadian2021linear, bhandari2021linear, xiao2022convergence} prove the linear convergence rate for the NPG method. Very recently, \citet{johnson2023optimal} shows that a linear convergence rate is optimal for NPG in tabular MDPs and \citet{mei2023ordering} provides a new perspective by proving a necessary and sufficient ordering-based condition for NPG convergence in bandit setting.


\paragraph{PG and PMD in regularized MDPs.} Another parallel line of work analyzes applying NPG method to maximum entropy reinforcement learning. \citet{cayci2021linear, cen2022fast} show that NPG with softmax policies can converge linearly in entropy-regularized MDPs while \citet{lan2023policy} also shows general PMD method converges linearly. Then, the linear convergence of PMD is extended to MDPs with general convex regularizers by \citet{zhan2023policy}. Meanwhile, \citet{li2022homotopic} and \citet{lan2023block} also propose other variants of PMD methods that converge linearly in entropy-regularized MDPs.

\paragraph{PG and PMD with function approximation.} 
\citet{agarwal2021theory} shows Q-NPG with log-linear policies achieves $O(1/\sqrt{K})$ convergence rate while \citet{cayci2021linear} and \citet{yuan2022linear} show that NPG with log-linear policies can converge linearly in entropy-regularized MDPs and unregularized MDPs. Meanwhile, \citet{chen2022finite} and \citet{chen2022sample} show similar $O(1/K)$ and linear convergence result under different assumptions, respectively. For more general function approximation setting, \citet{wang2019neural} shows that NPG with two-layer neural network has $O(1/\sqrt{K})$ convergence rate and \citet{liu2019neural} shows that NPG with multi-layer neural network achieves $O(1/\sqrt{K})$ convergence rate. Recently, \citet{alfano2023novel} shows PMD method with general function approximation can converge linearly. The main difference between \citet{alfano2023novel} and our work lies on how we define approximation, as discussed in Section \ref{sec:compare_ampo}.


\paragraph{Applications of PG.} Together with the rise of deep Q-learning \citep{mnih2013playing}, PG methods have also inspired many successful practical algorithms for real-world control task, including DDPG in \citet{lillicrap2015continuous}, TRPO in \citet{schulman2015trust}, PPO in \citet{schulman2017proximal} and SAC in \citet{haarnoja2018soft, haarnoja2018soft2}. Recently, \citet{tomar2020mirror} and \citet{vaswani2021general} propose general policy optimization algorithms based on mirror descent that are similar to ours. However, both of them treat policy parameterization as a black box and neither provides a convergence rate analysis.

\paragraph{Other related work.} The capability of policy gradient methods to do exploration in MDPs is also studied in \citet{cai2020provably, agarwal2020pc, shani2020optimistic, zanette2021cautiously}. \citet{grudzien2022mirror} proposes an abstract framework called mirror learning for both tabular and continuous-space MDPs that includes mirror descent as a special case. It provides an asymptotic convergence analysis but does not consider any function approximation setting. Finally, for optimization in functional space, \citet{chu2019probability} provides a framework setup that unifies variational inference and reinforcement learning. More recently, \citet{aubin2022mirror} studies mirror descent in general functional space and provides a rigorous convergence rate analysis. However, it only focuses on the primal space. 
\section{Legendre Function and Relaxations}
\label{sec:legendre}
Let~$\X$ be a normed vector space, possibly of infinite dimension,
and $\Phi:\X\to\R\cup\{+\infty\}$ a proper, closed convex function 
with 
$\dom\Phi=\{x\in\X \,|\, \Phi(x)<+\infty\}$. 
\begin{definition}\label{def:legendre}
The function $\Phi$ is of \textit{Legendre type} if
\begin{enumerate}[(a)]
\item 
The interior of $\dom\Phi$, 
denoted by $\mc{D}$, is nonempty;
\item $\Phi$ is differentiable and strictly convex on $\mc{D}$;
\item For any sequence $\Bp{x_n}\subset\mc{D}$ which converges to a boundary point of $\mc{D}$, it holds that $\lim_{n\rightarrow\infty}\Norm{\nabla\Phi(x_n)}=\infty$.
\end{enumerate}
\end{definition}

Let $\X^*$ be the dual vector space of~$\X$. 
The (Legendre) conjugate of $\Phi$ is defined as follows: 
for any $x^\star\in\X^\star$, 
\begin{equation}
    \label{equ:conjugate}
    \Phi^*(x^*)=\sup_{x\in\dom\Phi}\Bp{\inner{x, x^*}-\Phi(x)}.
\end{equation}
Similarly,
$\dom\Phi^*=\{x^*\in\X^* \,|\, \Phi^*(x^*)<+\infty\}$ 
and $\mc{D}^*=\intr(\dom\Phi^*)$.
If $\Phi$ is of Legendre type, 
then its gradient $\nabla\Phi$ is one-to-one from $\mc{D}$ to $\mc{D}^*$ and $\nabla\Phi^*=(\nabla\Phi)^{-1}$; in other words, for any $x\in\mc{D}$ and $x^*\in\mc{D}^*$,
\begin{equation}\label{eqn:conj-inv}
\nabla\Phi^*(\nabla\Phi(x))=x, \qquad 
\nabla\Phi(\nabla\Phi^*(x^*))=x^*.
\end{equation}
See \citet[Theorem~26.5]{Rockafellar1970book} for further details.

However, if $\Phi$ is not of Legendre type, then~\eqref{eqn:conj-inv} may not hold. In particular, this is the case if the $\dom\Phi$ is the simplex $\Delta = \{x\in\R^n_+ \,|\,\sum_i x_i=1\}$,
which has an empty interior. 
In fact, such functions are not even differentiable.
To see this, let $\Phi(x)=\phi(x)+\delta(x|\Delta)$ where~$\phi$ is convex and differentiable over $\R^n$, and  
$\delta(\cdot|\Delta)$ is the indicator function of~$\Delta$, i.e., 
$\delta(x|\Delta)=0$ if $x\in\Delta$ and $+\infty$ otherwise.
Then $\Phi$ is not a differentiable function. 
However, it is subdifferentiable with subdifferential 
\begin{equation}\label{eqn:subdiff}
\partial\Phi(x) = \{\nabla\phi(x) + c\ones \,|\,c\in\R\},
\end{equation}
where $\ones=\matenv{1 & \dots & 1}^\top$.
Given the importance of simplex in studying MDPs, we present the following relaxation of~\eqref{eqn:conj-inv}, which is crucial for our main results.

\begin{restatable}{lemma}{affinelem}
    \label{lem:affine}
    Suppose $\Phi(x) = \phi(x) + \delta(x|\mc{L})$ where $\phi$ is a convex function of Legendre type and $\mc{L}$ is an affine subspace. Assume that $\intr(\dom\phi)\cap\mc{L}\neq\emptyset$. Then we have
    \[
    \nabla\Phi^*(\nabla\Phi(x))=x, \qquad 
    \forall\,x\in\intr(\dom\phi)\cap\mc{L}.
    \]
    And for any $x^*\in\intr (\dom\Phi^*)$ and any $x,y\in\dom\Phi$, 
    \[
    \bigl\langle\nabla\Phi(\nabla\Phi^*(x^*)), x-y\bigr\rangle
    =\langle x^*, x-y\rangle ,
    \]
    where $\nabla\Phi(x)$ denotes any subgradient in $\partial\Phi(x)$.
\end{restatable}
\begin{proof}
    Let $\mc{L}=x_0+\mc{V}$ where $\mc{V}$ is a subspace, and denote $\mc{V}^{\perp}$ its orthogonal complement. First, it is commonly known that the subdifferential of an indicator function is a normal cone \citep{bertsekas2009convex}. Thus, we have $\partial\delta(x\mid\mc{L})=\mc{N}_{\mc{L}}(x)\overset{\text{def}}{=}\Bp{g'\mid \inner{g', v+x_0-x}\leq 0, \forall v\in\mc{V}}$. That is, for any $g'\in \mc{N}_{\mc{L}}(x)$, we have $\inner{g', v}\leq \inner{g', x-x_0}$ for any $v\in\mc{V}$. Since $\mc{V}$ is a subspace, for any $v\in\mc{V}$, we have $\alpha v\in\mc{V}$ for any $\alpha\in\R$. Therefore, we must have $\inner{g', v}=0$ for any $v\in\mc{V}$ and $g'\in \mc{N}_{\mc{L}}(x)$. That is, we have $\mc{N}_{\mc{L}}(x)=\mc{V}^{\perp}$. (The reverse side is straightforward.)
    
    Suppose $x\in\intr(\dom\phi)\cap\mc{L}$. Then, according to the subdifferential calculus rule, we must have
    $\nabla\Phi(x)=\nabla\phi(x) + \xi$ for some $\xi\in\mc{N}_{\mc{L}}(x)=\mc{V}^{\perp}$.
    Let $x'\triangleq\nabla\Phi^*(\nabla\Phi(x))$.
    By definition of $\Phi^*$ and strict convexity of $\phi$, we have
    \[
    x'= \argmax_{z\in\mc{L}\cap\dom\phi}\left\{\langle z, \nabla\phi(x)+\xi\rangle - \phi(z)\right\}.
    \]
    The optimality condition of the above problem is
    \[
    \nabla\phi(x)+\xi - \nabla\phi(x') \in \mc{N}_{\mc{L}\cap\dom\phi}(x')=\mc{V}^{\perp} + \mc{N}_{\dom\phi}(x'),
    \]
    where the last equality above holds because $\mc{N}_{\mc{L}}(x)=\mc{V}^{\perp}$. Note that $x'=\nabla\Phi^*(\nabla\Phi(x))$ implies $\nabla\Phi(x)\in\partial\Phi(x')=\partial\phi(x')+\partial\delta(x'|\mc{L})$. As shown in \citet{rockafellar1967conjugates}, $\partial\phi(x)=\emptyset$ for any $x\in\mathrm{bd}\ \dom\phi$ for a Legendre type function $\phi$. Therefore, we must have $x'\in\intr(\dom\phi)$, which then implies $\mc{N}_{\dom\phi}(x')=\Bp{\bm{0}}$. Thus, we have
    \[
    \nabla\phi(x)+\xi - \nabla\phi(x') \in \mc{V}^{\perp}
    \]
    
    Since $\xi\in\mc{V}^{\perp}$, we conclude that $\nabla\phi(x)-\nabla\phi(x')\in\mc{V}^{\perp}$.
    On the other hand, we have $x,x'\in\mc{L}$, which implies that $x-x'\in\mc{V}$.
    Therefore,
    \[
    \left\langle \nabla\phi(x)-\nabla\phi(x'),\,x-x'\right\rangle = 0.
    \]
    Since $\phi$ is strictly convex, we must have $x=x'$, thus proving $\nabla\Phi^*(\nabla\Phi(x))=x$.
    
    To prove the second statement, let $x'\triangleq\nabla\Phi^*(x^*)$, i.e.,
    \[
    x' = \argmax_{z\in\mc{L}\cap\dom\phi} \left\{\langle z, x^*\rangle-\phi(z)\right\}.
    \]
    By similar reasoning, we have $x'\in\intr(\dom\phi)$. Thus, the optimality condition is 
    $x^* - \nabla\phi(x')\in \mc{V}^{\perp}$, meaning
    $\nabla\phi(x')=x^* + \xi$ for some $\xi\in\mc{V}^{\perp}$.
    Meanwhile,
    \[
    \nabla\Phi(\nabla\Phi^*(x^*)) = \nabla\Phi(x') =\nabla\phi(x') + \xi'
    =x^* + \xi + \xi',
    \]
    where $\xi'\in\mc{V}^{\perp}$.
    Since $\xi,\xi'\in\mc{V}^{\perp}$ and $x-y\in\mc{V}$, we have
    \[
    \left\langle\nabla\Phi(\nabla\Phi^*(x^*)), x-y\right\rangle
    =\langle x^*, x-y\rangle .
    \]
    This finishes the proof.
\end{proof}


Notice that if $\dom \phi=\R^n_+$ and $\mc{L}=\{x\in\R^n|\ones^T x=1\}$, then $\dom\Phi=\dom \phi \cap \mc{L} = \Delta$.
This is how we will invoke Lemma~\ref{lem:affine} with $\phi$ being the negative entropy function. We call $\Phi(x)$ defined in Lemma \ref{lem:affine} as \textit{relaxed Legendre-type function}.

Furthermore, we have the following corollary so that the Bregman divergence in Eq. \eqref{equ:bregman} is also well-defined for the relaxed Legendre-type function.
\begin{corollary}
    \label{coro:breg_well_defined}
    In the setting of Lemma \ref{lem:affine}, for any $x, y, z\in\dom\Phi$ and $g\in\partial\Phi(z)$, we have
    $$\inner{g, x-y}=\inner{\nabla\phi(z), x-y},$$
    which makes expression $\inner{\nabla\Phi(z), x-y}$ well-defined for $x, y\in\dom\Phi$.
\end{corollary}
\begin{proof}
    Again let $\mc{L}=x_0+\mc{V}$. As shown in the proof of Lemma \ref{lem:affine}, we have $\partial\Phi(z)=\nabla\phi(z)+\partial\delta(z\mid\mc{L})$ and $\partial\delta(z\mid\mc{L})=\mc{V}^{\perp}$. Then, since $\dom\Phi\subseteq\mc{L}$, we have $x-y\in\mc{V}$, which means to have $\inner{g', x-y}=0$ for any $g'\in \partial\delta(z\mid\mc{L})$. Therefore, we have
    $$\inner{g, x-y}=\inner{\nabla\phi(z), x-y} + \inner{g', x-y}=\inner{\nabla\phi(z), x-y}.$$
\end{proof}

The following corollary gives a dual relationship between $\Phi$'s and $\Phi^*$'s Bregman divergences.
\begin{corollary}
    \label{coro:dual_bregman}
    In the setting of Lemma \ref{lem:affine}, for $x^*, y^*\in\mathrm{int}\Sp{\dom\Phi^*}$, we have
    $$D_{\Phi^*}(x^*, y^*)=D_{\Phi}(\nabla\Phi^*(y^*), \nabla\Phi^*(x^*)).$$
\end{corollary}
\begin{proof}
    By definition of the Bregman divergence in Eq. \eqref{equ:bregman}, we have
    \begin{align*}
        &D_{\Phi^*}(x^*, y^*)-D_{\Phi}(\nabla\Phi^*(y^*), \nabla\Phi^*(x^*))\\
        =& \Phi^*(x^*) - \Phi^*(y^*) - \inner{\nabla\Phi^*(y^*), x^* - y^*}\\
        &\qquad - \Mp{\Phi(\nabla\Phi^*(y^*)) - \Phi(\nabla\Phi^*(x^*)) - \inner{\nabla\Phi(\nabla\Phi^*(x^*)), \nabla\Phi^*(y^*) - \nabla\Phi^*(x^*)}}\\
        =& \Mp{\Phi^*(x^*) + \Phi(\nabla\Phi^*(x^*))} - \Mp{\Phi^*(y^*) + \Phi(\nabla\Phi^*(y^*))}\\
        &\qquad - \inner{\nabla\Phi^*(y^*), x^* - y^*} + \inner{x^*, \nabla\Phi^*(y^*) - \nabla\Phi^*(x^*)}\tag{By Lemma \ref{lem:affine} and $\nabla\Phi^*(y^*), \nabla\Phi^*(x^*)\in\dom\Phi$.}\\
        =& \inner{x^*, \nabla\Phi^*(x^*)} - \inner{y^*, \nabla\Phi^*(y^*)} - \inner{\nabla\Phi^*(y^*), x^* - y^*} + \inner{x^*, \nabla\Phi^*(y^*) - \nabla\Phi^*(x^*)}\tag{By \citet[Proposition~5.4.3]{bertsekas2009convex}.}\\
        =& \inner{\nabla\Phi^*(x^*), x^* - x^*} + \inner{\nabla\Phi^*(y^*), -y^* + y^* - x^* + x^*}\\
        =& 0.
    \end{align*}
\end{proof}
\section{Analysis of \algoacronym-$L_2$}
\label{sec:analysis-L2}
The analysis of \algoacronym-KL requires a slightly modified assumption (\textsf{A1'}) compared with \algoacronym-KL.
\begin{enumerate}
\item[(\textsf{A1'})] Under the same setting as (\textsf{A1}), we instead have
    $$\E_{s\sim d^{(k)}_\rho} \!
                \Bigl[\bigl\|f^{\bkone}_s\!-\bigl(\pi^{\bk}_s\! + \eta_k\widehat{Q}^{\bk}_s\bigr)\bigr\|_2^2\Bigr]\leq 2\eta_k^2\epsilon_{\actor}.$$
\end{enumerate}

The scaling coefficient $\eta_k^2$ is consistent with Assumption~\textsf{(A1')} 
in \citet{alfano2023novel} as they assume the $L_2$-error for approximating 
Meanwhile, notice that in (\textsf{A1}) at Eq. \eqref{equ:error_assumption}, the upper bound of the approximation error is $\eta_k\epsilon_{\actor}$ and is different from (\textsf{A1'}). 
This is the result of considering the growth rate of the approximation error in~$\eta_k$ for different Bregman divergences. 
Specifically, if we keep $f^{\bkone}$, $f^{\bk}$ and $\widehat{Q}^{\bk}$ fixed, then the $L_2$-error in (\textsf{A1'}) satisfies 
$$\E_{s\sim d^{(k)}_\rho}\Bigl[\bigl\|f^{\bkone}_s-\bigl(\pi^{\bk}_s + \eta_k\widehat{Q}^{\bk}_s\bigr)\bigr\|_2^2\Bigr]\propto\eta_k^2.$$
However, the KL-divergence in (\textsf{A1}) satisfies 
\[
    \E_{s\sim d^{(k)}_\rho} \!\! \Mp{D_{\KL}\!\Sp{\pi^{(k+1)}_s \Big\|\, \pi_s^{\bk}\exp\bigl(-\eta_k\widehat{Q}^{\bk}_s\bigr)/Z^{\bk}_s}}
    \propto\eta_k.
\]

Then, we have the following theorem.
\begin{restatable}[Linear Convergence of \algoacronym-$L_2$]{theorem}{convltwo}
    \label{theo:convergence_l2_bapmd}
    Consider Algorithm \ref{alg:compo} with initial policy $\pi^{(0)}$, initial distribution $\rho\in\Delta(\S)$ and $\Phi$ being the squared $L_2$-norm. Suppose Assumptions (\textsf{A1'}) and (\textsf{A2}) hold and the step sizes satisfy $\eta_0>1$ and 
$\eta_{k+1}\geq\left(\vartheta_{\rho}/(\vartheta_{\rho}-1)\right)\eta_k$ 
    for all $k\geq 0$. Then, for any comparator policy $\pi^{\star}$, it holds that
    \IfTwoColumnElse{
        \begin{align*}
            V^{(K)}_{\rho}-V^\star_\rho
            \leq&\; \biggl(1-\frac{1}{\vartheta_{\rho}}\biggr)^{\!\!K} \!\!
            \left(V^{(0)}_{\rho}-V^\star_\rho + \frac{D^\star_0 / (\vartheta_{\rho}-1)}{(1-\gamma)\eta_0} \right)\\
            \quad &  +\frac{\vartheta_{\rho}^2\sqrt{2\epsilon_{\actor}} + 2\vartheta_{\rho}\epsilon_{\critic}}{1-\gamma}.
        \end{align*}
    }{
        $$V^{(K)}_{\rho}-V^\star_\rho\leq\Sp{1-\frac{1}{\vartheta_{\rho}}}^K\left(V^{(0)}_{\rho}-V^\star_\rho + \frac{D^\star_0 / (\vartheta_{\rho}-1)}{(1-\gamma)\eta_0} \right)  +\frac{\vartheta_{\rho}^2\sqrt{2\epsilon_{\actor}} + 2\vartheta_{\rho}\epsilon_{\critic}}{1-\gamma}.$$
    }
     where $D_0^\star=\E_{s\sim d^\star_{\rho}}\bigl[D_{\Phi}(\pi_s^\star, \pi_s^{(0)})\bigr]$.
\end{restatable}
The proof of Theorem~\ref{theo:convergence_l2_bapmd} is given in Appendix \ref{sec:proof_general}. It retains the convergence rate of \citet{alfano2023novel} albeit with some different techniques. This is expected since in the $L_2$ case, \algoacronym-$L_2$ is the same as \AMPO. 
\section{Convergence Analysis of \algoacronym}
\label{sec:proof_general}

The analysis starts by proving an approximate version of the Pythagorean theorem, which controls the error in three-point identity by the corresponding Bregman divergence and will serve as the key tool of our analysis.

\subsection{Approximate Pythagorean Theorem}
We begin with a general upper bound and then, we will derive its extensions under specific choices of mirror maps.

\begin{lemma}[Approximate Pythagorean Theorem]
    \label{lem:modified_pytha}
    Let $\Phi:\C\mapsto\R$ be a proper closed convex mirror map, $\D\subseteq\C$ be a closed convex set and $v, c\in\C$ be two points. Suppose $u^\star=\argmin_{u\in\D}D_{\Phi}(u, v)$. Then, for any $u\in\D$, we have
    \IfTwoColumnElse{
        \begin{align*}
            &D_{\Phi}(u, u^\star) + D_{\Phi}(u^\star, c) - D_{\Phi}(u, c)\\
                \leq& \inner{\nabla\Phi(v)-\nabla\Phi(c), u^\star-u}.
        \end{align*}
    }{
        $$D_{\Phi}(u, u^\star) + D_{\Phi}(u^\star, c) - D_{\Phi}(u, c)\leq \inner{\nabla\Phi(v)-\nabla\Phi(c), u^\star-u}.$$
    }
\end{lemma}

\begin{proof}
    Using the definition of Bregman divergence in Eq. \eqref{equ:bregman}, we have
    \begin{align}
        D_{\Phi}(u, u^\star) + D_{\Phi}(u^\star, c) - D_{\Phi}(u, c)=&\inner{\nabla \Phi(u^\star), u-u^\star} - \inner{\nabla\Phi(c), u^\star-c} +\inner{\nabla\Phi(c), u-c}\nonumber\\
        =&\inner{\nabla\Phi(u^\star) - \nabla\Phi(c), u^\star - u }\label{equ:three_point_identity}\\
        =&\underbrace{\inner{\nabla\Phi(u^\star) - \nabla\Phi(v), u^\star - u}}_{\leq 0 \text{ by Lemma 4.1 in \citet{bubeck2012regret}}} + \inner{\nabla\Phi(v) - \nabla\Phi(c), u^\star-u}\nonumber\\
        \leq& \inner{\nabla\Phi(v) - \nabla\Phi(c), u^\star-u}.\nonumber
    \end{align}
\end{proof}

\subsubsection{Extension under Squared $L_2$-Norm}

\begin{lemma}
    \label{lem:l2_pytha}
    Under the condition of Lemma \ref{lem:modified_pytha}, if we take $\Phi$ to be the squared $L_2$-norm (see Example \ref{ex:squared-2norm}) and $\mc{D}=\Delta(\A)$, then for any $u\in\mc{D}$, we have
    $$D_{\Phi}(u, u^\star) + D_{\Phi}(u^\star, c) - D_{\Phi}(u, c)\leq \sqrt{2D_{\Phi}(v, c)}.$$
\end{lemma}

\begin{proof}
    By Lemma \ref{lem:modified_pytha}, we only need to bound $\inner{\nabla\Phi(v) - \nabla\Phi(c), u^\star-u}$. Then, since $\nabla\Phi(x)=x$, we have
    \begin{align*}
        \inner{\nabla\Phi(v)-\nabla\Phi(c), u^\star-u}\leq \Norm{v - c}_2\Norm{u^\star-u}_2\leq \sqrt{2 D_{\Phi}(v, c)},
    \end{align*}
    where $\Norm{u^\star-u}_2\leq 1$ since $u^\star, u\in\Delta(\A)$.
\end{proof}

\subsubsection{Extension under Negative Entropy}

\begin{lemma}
    \label{lem:kl_pytha}
    Under the condition of Lemma \ref{lem:modified_pytha}, if we take $\Phi$ to be the negative entropy restricted on $\Delta(\A)$ (see Example \ref{ex:ent-simplex}) and assume $\C=\mc{D}=\Delta(\A)$, then for any $u\in\mc{D}$, we have
    \IfTwoColumnElse{
        \begin{align*}
            &D_{\Phi}(u, u^\star) + D_{\Phi}(u^\star, c) - D_{\Phi}(u, c)\\
            \leq& \Sp{1+\Norm{\frac{u}{v}}_{\infty}}\Sp{D_{\Phi}(v, c)+\sqrt{2D_{\Phi}(v, c)}}.
        \end{align*}
    }{
        $$D_{\Phi}(u, u^\star) + D_{\Phi}(u^\star, c) - D_{\Phi}(u, c)\leq \Sp{1+\Norm{\frac{u}{v}}_{\infty}}\Sp{D_{\Phi}(v, c)+\sqrt{2D_{\Phi}(v, c)}}.$$
    }
\end{lemma}

\begin{proof}
    By Lemma \ref{lem:modified_pytha}, we only need to bound $\inner{\nabla\Phi(v) - \nabla\Phi(c), u^\star-u}$. Then, since $\C=\mc{D}=\Delta(\A)$, we have $v=u^\star$. Therefore, we have
    \begin{align*}
        \inner{\nabla\Phi(v)-\nabla\Phi(c), u^\star-u}=& \inner{\nabla\Phi(v)-\nabla\Phi(c), v-u}\\
        =&\inner{\log\frac{v}{c}, v-u}\tag{Since $\Phi$ is the negative Shannon entropy}\\
        =& D_{\KL}(v\|c) - \inner{\log\frac{v}{c}, u}\\
        \leq& D_{\KL}(v\|c) + \Norm{\frac{u}{v}}_{\infty}\inner{\abs{\log\frac{v}{c}}, v}\\
        \leq& D_{\KL}(v\|c) + \Norm{\frac{u}{v}}_{\infty}\Sp{D_{\KL}(v\| c)+\sqrt{2 D_{\KL}(v\| c)}}.\tag{By Lemma \ref{lem:absolute_kl_bound}}\\
        \leq& \Sp{1 + \Norm{\frac{u}{v}}_{\infty}}\Sp{D_{\Phi}(v, c)+\sqrt{2 D_{\Phi}(v, c)}}
    \end{align*}
\end{proof}

\begin{lemma}
    \label{lem:absolute_kl_bound}
    For any distributions $p, q\in\Delta(\A)$ such that $p$ is absolutely continuous with respect to $q$, we have $\inner{\abs{\log\frac{p}{q}}, p}\leq D_{\KL}(p\| q)+\sqrt{2D_{\KL}(p\| q)}$.
\end{lemma}
\begin{proof}
    Without loss of generality, assume $\mathrm{supp}(p)=\A$. Now, we define $\A^{+}=\Bp{a\in\A\mid p_a\geq q_a}$ and $\A^{-}=\Bp{a\in\A\mid p_a< q_a}$. Then, when $p, q$ are discrete distributions, we have
    \begin{align*}
        \inner{\abs{\log\frac{p}{q}}, p}=&\sum_{a\in\A^{+}}p_a\log\frac{p_a}{q_a} + \sum_{a\in\A^{-}}p_a\log\frac{q_a}{p_a}\\
        =& \sum_{a\in\A^{+}}p_a\log\frac{p_a}{q_a} - \sum_{a\in\A^{-}}p_a\log\frac{q_a}{p_a} + \sum_{a\in\A^{-}}p_a\log\frac{q_a}{p_a} + \sum_{a\in\A^{-}}p_a\log\frac{q_a}{p_a}\\
        =& D_{\KL}(p\| q) + 2\sum_{a\in\A^{-}}p_a\log\frac{q_a}{p_a}\\
        \leq& D_{\KL}(p\| q) + 2\sum_{a\in\A^{-}}p_a\Sp{\frac{q_a}{p_a}-1}\tag{Since $\log x\leq x-1$ for any $x>0$.}\\
        =& D_{\KL}(p\| q) + 2\sum_{a\in\A^{-}}\Sp{q_a-p_a}\\
        =& D_{\KL}(p\| q) + 2\Norm{q-p}_{\mathrm{TV}}\tag{By definition of total variation distance.}\\
        \leq& D_{\KL}(p\| q) + \sqrt{2D_{\KL}\Sp{p\| q}}.\tag{By Pinsker's inequality.}
    \end{align*}
\end{proof}

\subsection{Proof of Theorem \ref{theo:convergence_l2_bapmd} and Theorem \ref{theo:convergence_kl_bapmd}}
We first recall the Assumption (\textsf{A1}), (\textsf{A1'}), (\textsf{A2}) and (\textsf{A3}) listed in Section \ref{sec:convergence}. Here, we prove Theorem \ref{theo:convergence_l2_bapmd} and Theorem \ref{theo:convergence_kl_bapmd} in a slightly more general version in which the training data distributions can be different from $d^{\bk}_{\rho}$ as long as they satisfy the assumptions. This slight extension makes our result applicable to the offline training setting where $\nu^{\bk}\in\Delta(\S)$ is the reply buffer distribution at $k$-th iteration. Taking $\nu^{\bk}=d^{\bk}_{\rho}$ recovers the online training setting.




\begin{enumerate}
    \item[(\textsf{A1})] With initial distribution $\rho\in\Delta(\S)$ and reply buffer distribution $\nu^{\bk}\in\Delta(\S)$, there exist constants $\epsilon_{\critic}, \epsilon_{\actor}>0$ such that for any $k$, it holds
    \begin{align*}
        &\E_{s\sim \nu^{\bk}}\Mp{\Norm{\widehat{Q}^{\bk}_s-Q^{\bk}_s}_{\infty}}\leq\epsilon_{\critic},\\
        &\E_{s\sim \nu^{\bk}}\Mp{D_{\KL}\Sp{\pi^{\bkone}_s\left\|\frac{\pi_s^{\bk}\exp\Sp{-\eta_k\widehat{Q}^{\bk}_s}}{Z^{\bk}_s}\right.}}\leq \eta_k\epsilon_{\actor},
    \end{align*}

    \item[(\textsf{A1'})] Under the same setting as (\textsf{A1}), we instead have
    $$\E_{s\sim \nu^{\bk}}\Mp{\frac{1}{2}\Norm{f^{\bkone}_s-\Sp{\pi^{\bk}_s + \eta_k\widehat{Q}^{\bk}_s}}_2^2}\leq \eta_k^2\epsilon_{\actor}.$$

    \item[(\textsf{A2})] With initial distribution $\rho$ and replay buffer distribution $\nu^{\bk}\in\Delta(\S)$, there exists constant $\vartheta_{\rho}\geq 1$ such that for any $k$, it holds
    $$\max \Bp{\Norm{\frac{d^{\star}_{\rho}}{d^{\bkone}_{\rho}}}_{\infty}, \Norm{\frac{d^{\bkone}_{\rho}}{\nu^{\bk}}}_{\infty}, \Norm{\frac{d^{\bkone}_{d^\star_\rho}}{\nu^{\bk}}}_{\infty}, \Norm{\frac{d^{\bkone}_{d^\star_\rho}}{d^\star_\rho}}_{\infty}}\leq\vartheta_\rho.$$

    \item[(\textsf{A3})] There exists constant $C_{\rho}>0$ such that for any $k$, it holds
    $$\max_{s\in\mathrm{supp}(\nu^{\bk})}\Bp{\Norm{\frac{\pi^\star_s}{\pi^{\bkone}_s}}_{\infty}, \Norm{\frac{\pi^{\bk}_s}{\pi^{\bkone}_s}}_{\infty}}\leq C_{\rho}.$$
\end{enumerate}

To present the proof in an unified way, we further define $C_{\rho, s}=\max\Bp{\Norm{\frac{\pi^\star_s}{\pi^{\bkone}_s}}_{\infty}, \Norm{\frac{\pi^{\bk}_s}{\pi^{\bkone}_s}}_{\infty}}$ for some state $s\in\S$ and $\psi_s^{\Phi}:\R_{+}\mapsto\R_{+}$ as
\begin{equation}
    \label{equ:func_phi_s}
    \psi_s^{\Phi}(x)=\begin{cases}
    \sqrt{2x}, &\text{if $\Phi$ is squared $L_2$-norm},\\
    \Sp{1+C_{\rho, s}}\Sp{x+\sqrt{2x}},& \text{if $\Phi$ is the negative entropy on $\Delta(\A)$}.
\end{cases}
\end{equation}
Then, applying Lemma \ref{lem:l2_pytha} and \ref{lem:kl_pytha} to Algorithm \ref{alg:compo} will result the following key lemma.

\begin{lemma}
	\label{lem:base_relation}
	Consider running Algorithm \ref{alg:compo}. Then, for policy $\pi=\pi^{\bk}$ or $\pi=\pi^\star$, for any $s\in\S$, if $\Phi$ is either squared $L_2$-norm or negative entropy on $\Delta(\A)$, we have
        \begin{align*}
            &\eta_k\inner{\widehat{Q}^{\bk}_s, \pi^{\bkone}_s - \pi_s} + D_{\Phi}\Sp{\pi_s, \pi^{\bkone}_s} + D_{\Phi}\Sp{\pi_s^{\bkone}, \pi^{\bk}_s} - D_{\Phi}\Sp{\pi_s, \pi^{\bk}_s}\\
            \leq&\psi^{\Phi}_s\Sp{D_{\Phi}\Sp{\nabla\Phi^*(f^{\bkone}_s), \nabla\Phi^*\Sp{\nabla\Phi(\pi^{\bk}_s) - \eta_k\widehat{Q}^{\bk}_s}}}.
        \end{align*}
\end{lemma}
\begin{proof}
	Fix some $s\in\S$. Since line \ref{line:projection} of Algorithm \ref{alg:compo} states that
	$$\pi^{\bkone}_s\in\argmin_{\pi'_s\in\Delta(\mc{A})}D_\Phi\Sp{\pi'_s, \nabla \Phi^*(f_s^{(k+1)})},$$
	Then, We can apply Lemma \ref{lem:l2_pytha} or \ref{lem:kl_pytha} with $\D=\Delta(\A)$, $u=\pi_s$, $u^\star=\pi^{\bkone}_s$, $v=\nabla\Phi^*(f_{s}^{\bkone})$ and $c=\nabla\Phi^*\Sp{\nabla\Phi(\pi_s^{\bk})-\eta_k\widehat{Q}^{\bk}_s}$, which gives us
	\begin{align*}
		& D_{\Phi}\Sp{\pi^{\bkone}_s, \nabla\Phi^*\Sp{\nabla\Phi(\pi_s^{\bk})-\eta_k\widehat{Q}^{\bk}_s}} - D_{\Phi}\Sp{\pi_s, \nabla\Phi^*\Sp{\nabla\Phi(\pi_s^{\bk})-\eta_k\widehat{Q}^{\bk}_s}}\\
		& + D_{\Phi}\Sp{\pi_s, \pi^{\bkone}_s} \leq \psi_s^{\Phi}\Sp{D_{\Phi}\Sp{\nabla\Phi^*(f^{\bkone}_s), \nabla\Phi^*\Sp{\nabla\Phi(\pi^{\bk}_s) - \eta_k\widehat{Q}^{\bk}_s}}}.
	\end{align*}
	By using the identity in Eq. \eqref{equ:three_point_identity}, for the left-hand side of the above inequality, we have
	\begin{align*}
		\text{LHS}=&\inner{\nabla\Phi(\pi^{\bkone}_s) - \nabla\Phi\Sp{\nabla\Phi^*\Sp{\nabla\Phi(\pi_s^{\bk})-\eta_k\widehat{Q}^{\bk}_s}}, \pi^{\bkone}_s - \pi_s}\\
            =&\inner{\nabla\Phi(\pi^{\bkone}_s) - \Sp{\nabla\Phi(\pi_s^{\bk})-\eta_k\widehat{Q}^{\bk}_s}, \pi^{\bkone}_s - \pi_s}\tag{By Lemma \ref{lem:affine}.}\\
		=& \eta_k\inner{\widehat{Q}^{\bk}_s, \pi^{\bkone}_s - \pi_s} + \inner{\nabla\Phi(\pi^{\bkone}_s) - \nabla\Phi(\pi_s^{\bk}), \pi^{\bkone}_s - \pi_s}\\
		=& \eta_k\inner{\widehat{Q}^{\bk}_s, \pi^{\bkone}_s - \pi_s} + D_{\Phi}\Sp{\pi_s, \pi^{\bkone}_s} + D_{\Phi}\Sp{\pi_s^{\bkone}, \pi^{\bk}_s} - D_{\Phi}\Sp{\pi_s, \pi^{\bk}_s}\tag{By using the identity in Eq. \eqref{equ:three_point_identity} again on the second term above.}
	\end{align*}
	The proof is then complete by plugging this inequality back.
\end{proof}

Notice that the conditions in Assumption (\textsf{A1}) and (\textsf{A1'}) can unifiedly written as
$$\E_{s\sim \nu^{\bk}}\Mp{D_{\Phi}\Sp{\nabla\Phi^*(f^{\bkone}_s), \nabla\Phi^*\Sp{\nabla\Phi(\pi^{\bk}_s) - \eta_k\widehat{Q}^{\bk}_s}}}\leq \eta_k^{\omega^{\Phi}}\epsilon_{\actor},$$
where we define $\omega^{\Phi}$ as
$$\omega^{\Phi}=\begin{cases}
    2, &\text{if $\Phi$ is squared $L_2$-norm},\\
    1,& \text{if $\Phi$ is the negative entropy on }\Delta(\A).
\end{cases}$$
Therefore, we can summarize both Theorem \ref{theo:convergence_l2_bapmd} and \ref{theo:convergence_kl_bapmd} into the following theorem and present its proof.
\begin{theorem}
    \label{theo:unified_theorem}
    Consider Algorithm \ref{alg:compo} with initial policy $\pi^{(0)}$, initial distribution $\rho\in\Delta(\S)$ and $\Phi$ being either the squared $L_2$-norm or negative entropy on $\Delta(\A)$. Let Assumption (\textsf{A1}), (\textsf{A1'}), (\textsf{A2}), (\textsf{A3}) hold and suppose the learning rates satisfy $\eta_0 \geq 1$ and $\eta_{k+1}\geq\frac{\vartheta_{\rho}}{\vartheta_{\rho}-1}\eta_k$ for any $k\in[K]$. Then, for any comparator policy $\pi^{\star}$, with $D_0^\star=\E_{s\sim d^\star_{\rho}}\Mp{D_{\Phi}(\pi_s^\star, \pi_s^{(0)})}$, it holds that
    $$V^{(K)}_{\rho}-V^\star_\rho\leq\Sp{1-\frac{1}{\vartheta_{\rho}}}^K\left(V^{(0)}_{\rho}-V^\star_\rho + \frac{D^\star_0 / (\vartheta_{\rho}-1)}{(1-\gamma)\eta_0} \right)  +\frac{\vartheta_{\rho}^2\psi^{\Phi}(\epsilon_{\actor}) + 2\vartheta_{\rho}\epsilon_{\critic}}{1-\gamma},$$
    where we define $\psi^{\Phi}:\R_{+}\mapsto\R_{+}$ as
    \begin{equation}
        \label{equ:func_phi}
        \psi^{\Phi}(x)=\begin{cases}
        \sqrt{2x}, &\text{if $\Phi$ is $L_2$-norm square},\\
        \Sp{1+C_{\rho}}\Sp{x+\sqrt{2x}},& \text{if $\Phi$ is the negative entropy}.
    \end{cases}
    \end{equation}
\end{theorem}
\begin{proof}
    \textbf{\textit{Step 1.}} Fix some $s\in\S$ and $k<K$. First, by Lemma \ref{lem:base_relation} with $\pi_s=\pi^{\bk}_s$, we have
    \begin{align*}
        &\eta_k\inner{\widehat{Q}^{\bk}_s, \pi^{\bkone}_s - \pi^{\bk}_s} + D_{\Phi}\Sp{\pi^{\bk}_s, \pi^{\bkone}_s}\\
        \leq& \psi_s^{\Phi}\Sp{D_{\Phi}\Sp{\nabla\Phi^*(f^{\bkone}_s), \nabla\Phi^*\Sp{\nabla\Phi(\pi^{\bk}_s) - \eta_k\widehat{Q}^{\bk}_s}}},
    \end{align*}
    where we dropped the term $D_{\Phi}\Sp{\pi_s^{\bkone}, \pi^{\bk}_s}$ since it is always non-negative.
    
    Since $D_{\Phi}\Sp{\pi^{\bk}_s, \pi^{\bkone}_s}\geq 0$ as a Bregman divergence, we have
    \fontsize{9.5}{9.5}
    \begin{equation}
        \label{equ:mono_improvement}
        \Delta_s^{\bk}\overset{\text{def}}{=}\eta_k\inner{\widehat{Q}^{\bk}_s, \pi^{\bkone}_s - \pi^{\bk}_s} - \psi_s^{\Phi}\Sp{D_{\Phi}\Sp{\nabla\Phi^*(f^{\bkone}_s), \nabla\Phi^*\Sp{\nabla\Phi(\pi^{\bk}_s) - \eta_k\widehat{Q}^{\bk}_s}}}\leq 0.
    \end{equation}
    \normalsize
    Then, by using Lemma \ref{lem:base_relation} with $\pi_s=\pi^\star_s$, the comparator policy, and similarly dropping $D_{\Phi}\Sp{\pi_s^{\bkone}, \pi^{\bk}_s}$, we have
    \begin{align*}
        &\eta_k\inner{\widehat{Q}^{\bk}_s, \pi^{\bkone}_s - \pi^{\star}_s} + D_{\Phi}\Sp{\pi^{\star}_s, \pi^{\bkone}_s} - D_{\Phi}\Sp{\pi^{\star}_s, \pi^{\bk}_s}\\
    \leq& \psi_s^{\Phi}\Sp{D_{\Phi}\Sp{\nabla\Phi^*(f^{\bkone}_s), \nabla\Phi^*\Sp{\nabla\Phi(\pi^{\bk}_s) - \eta_k\widehat{Q}^{\bk}_s}}}.
    \end{align*}
    By adding and subtracting $\eta_k\inner{\widehat{Q}^{\bk}_s, \pi^{\bk}_s}$ together with some rearrangement, we have
    \begin{align*}
        &\eta_k\inner{\widehat{Q}^{\bk}_s, \pi^{\bkone}_s - \pi^{\bk}_s} - \psi_s^{\Phi}\Sp{D_{\Phi}\Sp{\nabla\Phi^*(f^{\bkone}_s), \nabla\Phi^*\Sp{\nabla\Phi(\pi^{\bk}_s) - \eta_k\widehat{Q}^{\bk}_s}}}\\
     &\qquad+ \eta_k\inner{\widehat{Q}^{\bk}_s, \pi^{\bk}_s - \pi^{\star}_s} \leq D_{\Phi}\Sp{\pi^{\star}_s, \pi^{\bk}_s} - D_{\Phi}\Sp{\pi^{\star}_s, \pi^{\bkone}_s}.
    \end{align*}
    Taking expectation on both sides with respect to distribution $d^{\star}_\rho$, we have
    \begin{equation}
        \label{equ:expect_first_time}
        \begin{split}
    	&\E_{s\sim d^{\star}_{\rho}}\Mp{\eta_k\inner{\widehat{Q}^{\bk}_s, \pi^{\bkone}_s - \pi^{\bk}_s} - \psi_s^{\Phi}\Sp{D_{\Phi}\Sp{\nabla\Phi^*(f^{\bkone}_s), \nabla\Phi^*\Sp{\nabla\Phi(\pi^{\bk}_s) - \eta_k\widehat{Q}^{\bk}_s}}}}\\
    	& \qquad +\eta_k \E_{s\sim d^\star_{\rho}}\Mp{\inner{\widehat{Q}^{\bk}_s, \pi^{\bk}_s - \pi^{\star}_s}} \leq D^\star_{k} - D^{\star}_{k+1},
        \end{split}
    \end{equation}
    where $D^{\star}_k=\E_{s\sim d^\star_{\rho}}\Mp{D_{\Phi}\Sp{\pi^\star_s, \pi^{\bk}_s}}$. 
    
    Then, for the first expectation above, we have
    \begin{align*}
        &\E_{s\sim d^\star_{\rho}}\Mp{\Delta_s^{\bk}}\\
        \geq& \Norm{\frac{d^{\star}_{\rho}}{d^{\bkone}_{\rho}}}_{\infty}\E_{s\sim d^{\bkone}_{\rho}}\Mp{\Delta_s^{\bk}}\tag{By Eq. \eqref{equ:mono_improvement}.}\\
        \geq& \vartheta_{\rho} \E_{s\sim d^{\bkone}_{\rho}}\Mp{\Delta_s^{\bk}}\tag{By Assumption (\textsf{A2}).}\\
        =& \eta_k \vartheta_{\rho} \E_{s\sim d^{\bkone}_{\rho}}\Mp{\inner{\widehat{Q}^{\bk}_s, \pi^{\bkone}_s - \pi^{\bk}_s}} \\
            &\quad - \vartheta_{\rho}^2\E_{s\sim \nu^{\bk}}\Mp{\psi_s^{\Phi}\Sp{D_{\Phi}\Sp{\nabla\Phi^*(f^{\bkone}_s), \nabla\Phi^*\Sp{\nabla\Phi(\pi^{\bk}_s) - \eta_k\widehat{Q}^{\bk}_s}}}}\tag{By Assumption (\textsf{A2}).}\\
        \geq& \eta_k \vartheta_{\rho} \E_{s\sim d^{\bkone}_{\rho}}\Mp{\inner{Q^{\bk}_s, \pi^{\bkone}_s - \pi^{\bk}_s}} + \eta_k\vartheta_{\rho} \E_{s\sim d^{\bkone}_{\rho}}\Mp{\inner{\widehat{Q}^{\bk}_s - Q^{\bk}_s, \pi^{\bkone}_s - \pi^{\bk}_s}}\\
    	&\quad - \vartheta_{\rho}^2\psi^{\Phi}\Sp{\E_{s\sim \nu^{\bk}}\Mp{D_{\Phi}\Sp{\nabla\Phi^*(f^{\bkone}_s), \nabla\Phi^*\Sp{\nabla\Phi(\pi^{\bk}_s) - \eta_k\widehat{Q}^{\bk}_s}}}}\tag{By Assumption (\textsf{A3}), Jensen's inequality and concavity of $\psi_s^{\Phi}$.}\\
        \overset{\text{(i)}}{\geq}& \eta_k \vartheta_{\rho} \E_{s\sim d^{\bkone}_{\rho}}\Mp{\inner{Q^{\bk}_s, \pi^{\bkone}_s - \pi^{\bk}_s}} - \eta_k\vartheta_{\rho}\epsilon_{\critic} - \vartheta_{\rho}^2\psi^{\Phi}\Sp{\eta_k^{\omega^{\Phi}}\epsilon_{\actor}}\tag{By Assumption (\textsf{A1}), (\textsf{A1'}) and monotonicity of $\psi^{\Phi}$.}\\
        =&\Sp{1-\gamma}\eta_k \vartheta_{\rho}\Sp{V_{\rho}^{\bkone} - V_{\rho}^{\bk}} - \eta_k\vartheta_{\rho}\epsilon_{\critic} - \vartheta_{\rho}^2\psi^{\Phi}\Sp{\eta_k^{\omega^{\Phi}}\epsilon_{\actor}}\tag{By Lemma \ref{lem:performance_diff}.}
    \end{align*}
    The inequality (i) above holds also because by H\"{o}lder's inequality, we have 
    $$\inner{\widehat{Q}^{\bk}_s - Q^{\bk}_s, \pi^{\bkone}_s - \pi^{\bk}_s}\leq\Norm{\widehat{Q}^{\bk}_s - Q^{\bk}_s}_{\infty}\Norm{\pi^{\bkone}_s - \pi^{\bk}_s}_1\leq \Norm{\widehat{Q}^{\bk}_s - Q^{\bk}_s}_{\infty}.$$
    
    Similarly, for the second expectation in Eq. \eqref{equ:expect_first_time}, we have
    \begin{align*}
        &\eta_k \E_{s\sim d^\star_{\rho}}\Mp{\inner{\widehat{Q}^{\bk}_s, \pi^{\bk}_s - \pi^{\star}_s}}\\
        =& \eta_k \E_{s\sim d^\star_{\rho}}\Mp{\inner{Q^{\bk}_s, \pi^{\bk}_s - \pi^{\star}_s}} + \eta_k \E_{s\sim d^\star_{\rho}}\Mp{\inner{\widehat{Q}^{\bk}_s - Q^{\bk}_s, \pi^{\bk}_s - \pi^\star_s}}\\
        \geq& (1-\gamma)\eta_k\Sp{V^{\bk}_{\rho} - V^{\star}_{\rho}} - \eta_k\vartheta_{\rho}\epsilon_{\critic}.
    \end{align*}
    
    By plugging the results above back into Eq. \eqref{equ:expect_first_time} and defining $\delta_k\overset{\text{def}}{=}V^{\bk}_{\rho} - V^{\star}_{\rho}$, we have
    \begin{equation}
        \label{equ:recurrence_relation}
        \vartheta_{\rho}\Sp{\delta_{k+1} - \delta_k} + \delta_k \leq \frac{1}{(1-\gamma)\eta_k}D^\star_k - \frac{1}{(1-\gamma)\eta_k}D^\star_{k+1} + \frac{\vartheta_{\rho}^2\psi^{\Phi}\Sp{\eta_k^{\omega^{\Phi}}\epsilon_{\actor}}}{(1-\gamma)\eta_k} + \frac{2\vartheta_{\rho}\epsilon_{\critic}}{1-\gamma}.
    \end{equation}
    
    \textbf{\textit{Step 2.}} Now, dividing both sides of Eq. \eqref{equ:recurrence_relation} by $\vartheta_{\rho}$ together with some rearrangement, we can have
    $$\delta_{k+1}+\frac{D^\star_{k+1}}{(1-\gamma)\eta_k\vartheta_{\rho}}\leq\Sp{1-\frac{1}{\vartheta_{\rho}}}\Sp{\delta_k+\frac{D^\star_k}{(1-\gamma)\eta_k\Sp{\vartheta_{\rho}-1}}} + \frac{\vartheta_{\rho}\psi^{\Phi}\Sp{\eta_k^{\omega^{\Phi}}\epsilon_{\actor}}}{(1-\gamma)\eta_k} +  \frac{2\vartheta_{\rho}\epsilon_{\critic}}{(1-\gamma)\vartheta_{\rho}}.$$
    Since the learning rates satisfy $\eta_{k+1}(\vartheta_{\rho}-1)\geq\eta_k\vartheta_{\rho}$, we have
    \begin{align*}
        &\delta_{k+1}+\frac{D^\star_{k+1}}{(1-\gamma)\eta_{k+1}(\vartheta_{\rho}-1)}\\
        \leq&\Sp{1-\frac{1}{\vartheta_{\rho}}}\Sp{\delta_k+\frac{D^\star_k}{(1-\gamma)\eta_k\Sp{\vartheta_{\rho}-1}}} + \frac{\vartheta_{\rho}\psi^{\Phi}\Sp{\eta_k^{\omega^{\Phi}}\epsilon_{\actor}}}{(1-\gamma)\eta_k} +  \frac{2\vartheta_{\rho}\epsilon_{\critic}}{(1-\gamma)\vartheta_{\rho}}\\
        \leq& \Sp{1-\frac{1}{\vartheta_{\rho}}}\Sp{\delta_k+\frac{D^\star_k}{(1-\gamma)\eta_k\Sp{\vartheta_{\rho}-1}}} + \frac{\vartheta_{\rho}\psi^{\Phi}\Sp{\epsilon_{\actor}}}{(1-\gamma)} +  \frac{2\vartheta_{\rho}\epsilon_{\critic}}{(1-\gamma)\vartheta_{\rho}},
    \end{align*}
    where the second inequality above holds because we can straightforwardly verify that $\frac{\psi^{\Phi}(\eta_k^{\omega^{\Phi}}\epsilon_{\actor})}{\eta_k}\leq\psi^{\Phi}(\epsilon_{\actor})$ for either choice of $\Phi$. Then, applying the above relation recursively, we have
    \begin{equation}
        \label{equ:recursion_result}
        \begin{split}
            \delta_{K} + &\frac{D^\star_K}{(1-\gamma)\eta_K(\vartheta_{\rho}-1)}\leq\Sp{1-\frac{1}{\vartheta_{\rho}}}^K\Sp{\delta_0+\frac{D^\star_0}{(1-\gamma)\eta_0(\vartheta_{\rho}-1)}}\\
            &\qquad + \frac{\vartheta_{\rho}\psi^{\Phi}\Sp{\epsilon_{\actor}}}{1-\gamma}\sum_{k=0}^{K-1}\Sp{1-\frac{1}{\vartheta_{\rho}}}^k + \frac{2\vartheta_{\rho}\epsilon_{\critic}}{(1-\gamma)\vartheta_{\rho}}\sum_{k=0}^{K-1}\Sp{1-\frac{1}{\vartheta_{\rho}}}^k.
        \end{split}
    \end{equation}
    We can notice that
    $$\sum_{k=0}^{K-1}\Sp{1-\frac{1}{\vartheta_{\rho}}}^k\leq\frac{1}{1-\Sp{1-\frac{1}{\vartheta_{\rho}}}}=\vartheta_{\rho}.$$
    Therefore, dropping the term with $D^\star_K$ in Eq. \eqref{equ:recursion_result}, we can finally have
    $$V^{(K)}_{\rho}-V^\star_\rho\leq\Sp{1-\frac{1}{\vartheta_{\rho}}}^K\left(V^{(0)}_{\rho}-V^\star_\rho + \frac{D^\star_0 / (\vartheta_{\rho}-1)}{(1-\gamma)\eta_0} \right)  +\frac{\vartheta_{\rho}^2\psi^{\Phi}(\epsilon_{\actor}) + 2\vartheta_{\rho}\epsilon_{\critic}}{1-\gamma}.$$
\end{proof}

Then, Theorem \ref{theo:convergence_l2_bapmd} and \ref{theo:convergence_kl_bapmd} are immediate consequences of Theorem \ref{theo:unified_theorem}.

\convltwo*

\begin{proof}
    Apply $\psi^{\Phi}(x)=\sqrt{2x}$ to Theorem \ref{theo:unified_theorem}.
\end{proof}

\convkl*

\begin{proof}
    Apply $\psi^{\Phi}(x)=(1+C_{\rho})\Sp{x+\sqrt{2x}}$ to Theorem \ref{theo:unified_theorem}.
\end{proof}

\subsection{Technical Lemmas}

\begin{lemma}
    \label{lem:dis_mismatch_bound}
    If we take $\rho=\textsf{Unif}(\S)$, then for any $k$ and any policy $\pi$, it holds that $\Norm{\frac{d^{\pi}_{\rho}}{d^{\bk}_{\rho}}}_{\infty}\leq\frac{\abs{\S}}{1-\gamma}$.
\end{lemma}
\begin{proof}
    By definition of state-visitation distribution in Eq. \eqref{equ:def_visitation}, we can immediately get $d^{\bk}_{\rho, s}\geq (1-\gamma)\rho_s$ for any $s\in\S$ by truncating all terms with $t\geq 1$. Since $\rho=\textsf{Unif}(\S)$, we have $\rho_s=\frac{1}{\abs{\S}}$ for any $s\in\S$. Thus, we have
    $$\Norm{\frac{d^{\pi}_{\rho}}{d^{\bk}_{\rho}}}_{\infty}=\max_{s\in\S}\frac{d^{\pi}_{\rho, s}}{d^{\bk}_{\rho, s}}\leq\frac{1}{(1-\gamma)\rho_s}\leq\frac{\abs{\S}}{1-\gamma}.$$
\end{proof}

\begin{lemma}[Performance Difference Lemma]
	\label{lem:performance_diff}
	For any two policies $\pi, \tilde{\pi}:\S\mapsto\Delta(\A)$ and initial distribution $\rho\in\Delta(\S)$, it holds that
	$$V_{\rho}^{\pi} - V_{\rho}^{\tilde{\pi}} = \frac{1}{1-\gamma}\E_{s\sim d^{\pi}_{\rho}}\Mp{\inner{Q^{\tilde{\pi}}_s, \pi_s - \tilde{\pi}_s}}=\frac{1}{1-\gamma}\E_{s\sim d^{\tilde{\pi}}_{\rho}}\Mp{\inner{Q^{\pi}_s, \pi_s - \tilde{\pi}_s}}.$$
\end{lemma}
\begin{proof}
	See Lemma 1 in \citet{xiao2022convergence}.
\end{proof}

\section{Convergence Analysis of SAC}
\label{sec:proof_sac}

In this section, we prove a sublinear convergence rate for SAC under general function approximation by using our framework. It essentially adopts our proof techniques in Theorem \ref{theo:unified_theorem} to an entropy-regularized objective.

We start by presenting a modified version of the performance difference lemma under entropy-regularized reinforcement learning.

\mperdiff*

\begin{proof}
    By definition of the value function, we have
    \begin{align*}
        V_{\tau, \rho}^\pi-V_{\tau, \rho}^{\tilde{\pi}}=& \E_{a_t\sim \pi_{s_t}}\Mp{\sum_{t=0}^{\infty}\gamma^t\Sp{c(s_t, a_t)+\tau\log\pi(a_t\mid s_t)} \mmid s_0\sim\rho} - V_{\tau, \rho}^{\tilde{\pi}}\\
        =& \E_{a_t\sim \pi_{s_t}}\Mp{\sum_{t=0}^{\infty}\gamma^t\Sp{c(s_t, a_t)+\tau\log\pi(a_t\mid s_t)+\gamma V_{\tau}^{\tilde{\pi}}(s_{t+1})-V_{\tau}^{\tilde{\pi}}(s_t)}\mmid s_0\sim\rho}\\
        =& \E_{a_t\sim \pi_{s_t}}\Mp{\sum_{t=0}^{\infty}\gamma^t\Sp{Q_\tau^{\tilde{\pi}}(s_t, a_t)-V_\tau^{\tilde{\pi}}(s_t)+\tau\log\frac{\pi(a_t\mid s_t)}{\tilde{\pi}(a_t\mid s_t)}}\mmid s_0\sim\rho}\\
        =& \E_{a_t\sim \pi_{s_t}}\Mp{\sum_{t=0}^{\infty}\gamma^t\Sp{A_\tau^{\tilde{\pi}}(s_t, a_t) + \tau\log\frac{\pi(a_t\mid s_t)}{\tilde{\pi}(a_t\mid s_t)}}\mmid s_0\sim\rho}\tag{We define $A_\tau^\pi(s, a)=Q_\tau^\pi(s, a)-V_\tau^\pi(s)$.}\\
        =& \frac{1}{1-\gamma}\E_{s\sim d_{\rho}^{\pi}}\Mp{\inner{Q_{\tau, s}^{\tilde{\pi}}, \pi_s-\tilde{\pi}_{s}}+\tau D_{\KL}(\pi_{s}\|\tilde{\pi}_{s})}.
    \end{align*}
    Then, by similarly expanding the term $V_{\tau, \rho}^{\tilde{\pi}}$, we can get
    $$V_{\tau, \rho}^{\pi} - V_{\tau, \rho}^{\tilde{\pi}} = \frac{1}{1-\gamma}\E_{s\sim d^{\tilde{\pi}}_{\rho}}\Mp{\inner{Q^{\pi}_{\tau, s}, \pi_s - \tilde{\pi}_s} - \tau D_{\KL}(\tilde{\pi}_s \| \pi_{s})}.$$
    Thus, the proof is complete.
\end{proof}

Now, we start to prove Theorem \ref{theo:convergence_sac}. Here, for simplicity, we keep using the function $\psi_s$ and $\psi$ defined in Eq. \eqref{equ:func_phi_s} and \eqref{equ:func_phi}. However, we ignore the superscript $\Phi$ since in this concrete example, we only take $\Phi$ to be the negative entropy.

\convsac*

\begin{proof}
    First, it is straightforward to check that Lemma \ref{lem:base_relation} still holds under entropy-regularized reinforcement learning. Then, fix some $s\in\S$ and $k<K$, similar to the proof of Theorem \ref{theo:unified_theorem}, just like Eq. \eqref{equ:mono_improvement}, we also have
    \begin{equation}
        \label{equ:sac_mono_improvement}
        \Delta^{\bk}_{\tau, s}\overset{\text{def}}{=} \eta\inner{\widehat{Q}^{\bk}_{\tau, s}, \pi^{\bkone}_s - \pi^{\bk}_s} - \psi_s\Sp{D_{\KL}\Sp{\pi^{\bkone}_s\ \middle\|\ \pi_s^{\bk}\exp\Sp{-\eta\widehat{Q}^{\bk}_{\tau, s}}/Z^{\bk}_s}} \leq 0.
    \end{equation}
    
    Then, by using Lemma \ref{lem:base_relation} with $\pi_s=\pi^\star_s$, the comparator policy, we have
    \begin{align*}
        &\eta\inner{\widehat{Q}^{\bk}_{\tau, s}, \pi^{\bkone}_s - \pi^{\star}_s} + D_{\KL}\Sp{\pi^{\star}_s\ \middle\|\ \pi^{\bkone}_s} + D_{\KL}\Sp{\pi_s^{\bkone}\ \middle\|\ \pi^{\bk}_s} - D_{\KL}\Sp{\pi^{\star}_s\ \middle\|\ \pi^{\bk}_s}\\
    \leq& \psi_s\Sp{D_{\KL}\Sp{\pi^{\bkone}_s\ \middle\|\ \pi_s^{\bk}\exp\Sp{-\eta\widehat{Q}^{\bk}_{\tau, s}}/Z^{\bk}_s}}.
    \end{align*}
    Notice here the key difference to the proof of Theorem \ref{theo:unified_theorem} is that we do \textbf{not} drop the term $D_{\KL}\Sp{\pi_s^{\bkone}\ \middle\|\ \pi^{\bk}_s}$.

    By some algebraic rearrangement and taking expectation with respect to distribution $d^\star_{\rho}$, we then get
    \begin{equation}
        \label{equ:sac_expect_first_time}
        \begin{split}
    	&\E_{s\sim d^{\star}_{\rho}}\Mp{\eta\inner{\widehat{Q}^{\bk}_{\tau, s}, \pi^{\bkone}_s - \pi^{\bk}_s} - \psi_s\Sp{D_{\KL}\Sp{\pi^{\bkone}_s\ \middle\|\ \pi_s^{\bk}\exp\Sp{-\eta\widehat{Q}^{\bk}_{\tau, s}}/Z^{\bk}_s}}}\\
    	& \qquad +\eta \E_{s\sim d^\star_{\rho}}\Mp{\inner{\widehat{Q}^{\bk}_{\tau, s}, \pi^{\bk}_s - \pi^{\star}_s}} + \E_{s\sim d^\star_{\rho}}\Mp{D_{\KL}\Sp{\pi_s^{\bkone}\ \middle\|\ \pi^{\bk}_s}} \leq D^\star_{k} - D^{\star}_{k+1},
        \end{split}
    \end{equation}
    where $D^{\star}_k=\E_{s\sim d^\star_{\rho}}\Mp{D_{\KL}\Sp{\pi^\star_s\ \middle\|\ \pi^{\bk}_s}}$.

    For the first expectation above, we have
    \begin{align*}
        &\E_{s\sim d^\star_\rho}\Mp{\Delta^{\bk}_{\tau, s}}\\
        \overset{\text{(i)}}{\geq}& \frac{1}{1-\gamma}\E_{s\sim d^{\bkone}_{d^\star_\rho}}\Mp{\Delta^{\bk}_{\tau, s}}\\
        =& \frac{\eta}{1-\gamma}\E_{s\sim d^{\bkone}_{d^\star_\rho}}\Mp{\inner{Q^{\bk}_{\tau, s}, \pi^{\bkone}_s - \pi^{\bk}_s}} + \frac{\eta}{1-\gamma}\E_{s\sim d^{\bkone}_{d^\star_\rho}}\Mp{\inner{\widehat{Q}^{\bk}_{\tau, s} - Q^{\bk}_{\tau, s}, \pi^{\bkone}_s - \pi^{\bk}_s}} \\
        &\qquad - \frac{\vartheta_{\rho}}{1-\gamma}\E_{s\sim\nu^{\bk}}\Mp{\psi_s\Sp{D_{\KL}\Sp{\pi^{\bkone}_s\ \middle\|\ \pi_s^{\bk}\exp\Sp{-\eta\widehat{Q}^{\bk}_{\tau, s}}/Z^{\bk}_s}}}\tag{By Assumption (\textsf{A2})}\\
        \geq& \frac{\eta}{1-\gamma}\E_{s\sim d^{\bkone}_{d^\star_\rho}}\Mp{\inner{Q^{\bk}_{\tau, s}, \pi^{\bkone}_s - \pi^{\bk}_s}} - \frac{\eta\epsilon_{\critic}}{1-\gamma} - \frac{\vartheta_{\rho}}{1-\gamma}\psi(\eta\epsilon_{\actor})\tag{By Assumption (\textsf{A1}), (\textsf{A3}) and concavity of $\psi_s$.}\\
        =&\eta\Sp{V^{\bkone}_{\tau, d^\star_\rho} - V^{\bk}_{\tau, d^\star_\rho}} - \eta\tau\E_{s\sim d^{\bkone}_{d^\star_\rho}}\Mp{D_{\KL}\Sp{\pi^{\bkone}_s\ \middle\|\ \pi^{\bk}_s}} - \frac{\eta \vartheta_\rho \epsilon_{\critic}}{1-\gamma} - \frac{\vartheta_{\rho}\psi(\eta\epsilon_{\actor})}{1-\gamma}\tag{By Lemma \ref{lem:sac_performance_diff}, the modified performance difference lemma.}\\
        \geq& \eta\Sp{V^{\bkone}_{\tau, d^\star_\rho} - V^{\bk}_{\tau, d^\star_\rho}} - \eta\tau\vartheta_\rho\E_{s\sim d^{\star}_{\rho}}\Mp{D_{\KL}\Sp{\pi^{\bkone}_s\ \middle\|\ \pi^{\bk}_s}} - \frac{\eta \vartheta_\rho \epsilon_{\critic}}{1-\gamma} - \frac{\vartheta_{\rho}\psi(\eta\epsilon_{\actor})}{1-\gamma}
    \end{align*}
    Here, the above inequality (i) holds because Eq. \eqref{equ:sac_mono_improvement} holds and we have $d^{\bkone}_{d^\star_\rho, s}\geq (1-\gamma)d^\star_{\rho, s}$ for any $s\in\S$ as introduced in Section \ref{sec:prelim}.

    Then, for the second expectation in Eq. \eqref{equ:sac_expect_first_time}, we can similarly apply Lemma \ref{lem:sac_performance_diff} and obtain
    \begin{align*}
        \eta \E_{s\sim d^\star_{\rho}}\Mp{\inner{\widehat{Q}^{\bk}_{\tau, s}, \pi^{\bk}_s - \pi^{\star}_s}} =& \eta \E_{s\sim d^\star_{\rho}}\Mp{\inner{Q^{\bk}_{\tau, s}, \pi^{\bk}_s - \pi^{\star}_s}} + \eta \E_{s\sim d^\star_{\rho}}\Mp{\inner{\widehat{Q}^{\bk}_{\tau, s} - Q^{\bk}_{\tau, s}, \pi^{\bk}_s - \pi^\star_s}}\\
        \geq& (1-\gamma)\eta\Sp{V^{\bk}_{\tau, \rho} - V^{\star}_{\tau, \rho}} + (1-\gamma)\eta\tau D^\star_k - \eta\vartheta_{\rho}\epsilon_{\critic}.
    \end{align*}

    By plugging these bounds back into Eq. \eqref{equ:sac_expect_first_time}, we then have
    \begin{align*}
        (1-\gamma)\Sp{V^{\bk}_{\tau, \rho} - V^{\star}_{\tau, \rho}}\leq & \frac{D^\star_k}{\eta} - \frac{D^\star_{k+1}}{\eta} + V^{\bk}_{\tau, d^\star_\rho} - V^{\bkone}_{\tau, d^\star_\rho} + \frac{(2-\gamma)\vartheta_\rho\epsilon_{\critic}}{1-\gamma} + \frac{\vartheta_{\rho}\psi(\eta\epsilon_{\actor})}{(1-\gamma)\eta} \\
        &\quad + \Sp{\tau\vartheta_\rho - \frac{1}{\eta}}\E_{s\sim d^\star_\rho}\Mp{D_{\KL}\Sp{\pi^{\bkone}_s\ \middle\|\ \pi^{\bk}_s}} - (1-\gamma)\tau D^\star_k\\
        \leq& \frac{D^\star_k}{\eta} - \frac{D^\star_{k+1}}{\eta} + V^{\bk}_{\tau, d^\star_\rho} - V^{\bkone}_{\tau, d^\star_\rho} + \frac{(2-\gamma)\vartheta_\rho\epsilon_{\critic}}{1-\gamma} + \frac{\vartheta_{\rho}\psi(\eta\epsilon_{\actor})}{(1-\gamma)\eta}\tag{By taking $\eta\leq \frac{1}{\tau\vartheta_\rho}$ and noticing KL divergence is non-negative.}
    \end{align*}
    Finally, by noticing that $\frac{\psi(\eta\epsilon_{\actor})}{\eta}\leq\psi(\epsilon_{\actor})$ and by taking sum from $k=0$ to $K-1$, we can get
    $$\frac{1}{K}\sum_{k=0}^{K-1}\Sp{V^{\bk}_{\tau, \rho} - V^{\star}_{\tau, \rho}}\leq\frac{1}{K}\Sp{\frac{D^\star_0}{(1-\gamma)\eta}+\frac{V^{(0)}_{\tau, d^\star_\rho}}{1-\gamma}} + \frac{\vartheta_{\rho}\psi(\epsilon_{\actor})+ (2-\gamma)\vartheta_\rho\epsilon_{\critic}}{(1-\gamma)^2}.$$
\end{proof}

\section{Implementation Details}
\label{sec:implementation}

\subsection{Algorithm Details}

The implementations of \algoacronym-KL, AMPO and MAMPO are based on modifying the actor loss in SAC while keeping other parts unchanged. Therefore, we will first present the pseudocode of SAC and then give modified actor losses for \algoacronym-KL, AMPO and MAMPO.

\subsubsection{SAC}
The pseudocode of SAC is given in Algorithm \ref{algo:original_sac}.

Here, $J(\tau, \theta)$ and $J_q(\phi_i, \mc{B}, \phi_{\targ, 1}, \phi_{\targ, 2})$ in line \ref{line:update_tau} and \ref{line:update_Q} represent the loss functions to update regularization parameter $\tau$ and q-value networks, respectively. More details of these two loss functions can be found in \citet{haarnoja2018soft}.

\subsubsection{\algoacronym-KL}

To implement \algoacronym-KL, we will basically replace the update rule in \ref{line:update_pi} of Algorithm \ref{algo:original_sac} by \algoacronym-KL's update rule. To do this, we first need to rewrite \algoacronym-KL's update rule in Eq. \eqref{equ:sac_like_loss_regularize} in terms of $q^{\pi}_{\tau}=Q^{\pi}_{\tau}(s, a)+\tau\log\pi(a\mid s)$. In particular, we have
\begin{align*}
    \theta^{\bkone}\in&\argmin_{\theta}\E_{s\sim d^{\bk}_{\rho}}\Mp{D_{\KL}\Sp{\pi^{\theta}_s\ \left\|\  \frac{\pi^{\bk}_s\exp\Sp{\eta_k Q^{\bk}_{\tau, s}}}{Z^{\bk}_s}\right.}}\\
    =&\argmin_{\theta}\E_{s\sim d^{\bk}_{\rho}}\Mp{D_{\KL}\Sp{\pi^{\theta}_s\ \left\|\  \frac{\pi^{\bk}_s\exp\Sp{\eta_k q^{\bk}_{\tau, s}-\eta_k\tau\log\pi^{\bk}_s}}{Z^{\bk}_s}\right.}}\\
    =&\argmin_{\theta} \underset{\substack{s\sim d^{\bk}_{\rho} \\ a\sim\pi^{\theta}_s}}{\E} \Mp{\log\pi^{\theta}(a\mid s) - (1-\eta_k\tau)\log\pi^{\bk}(a\mid s) - \eta_k q^{\bk}_{\tau}(s, a)}\tag{By ignoring normalization constants.}\\
    =&\argmin_{\theta} \underset{\substack{s\sim d^{\bk}_{\rho} \\ a\sim\pi^{\theta}_s}}{\E} \Mp{\tau\log\pi^{\theta}(a\mid s) - (1-\beta)\tau\log\pi^{\bk}(a\mid s) - \beta q^{\bk}_{\tau}(s, a)},
\end{align*}
where $\beta\overset{\mathrm{def}}{=}\eta_k\tau<1$. We can see that the update rule exactly becomes the SAC's update rule when $\beta=1$, which means to have $\eta_k=\frac{1}{\tau}$, consistent with our derivation in Section \ref{sec:equivalence}.

\begin{algorithm}[ht]
    \caption{Soft Actor-Critic (SAC) \citep{haarnoja2018soft}}
    \label{algo:original_sac}
    \begin{algorithmic}[1]
        \STATE \textbf{Input:} Initial policy network parameter $\theta$; initial q-value network parameters $\phi_1$, $\phi_2$; replay buffer $\mc{D}$; learning rates $\lambda_q$, $\lambda_\pi$, $\lambda$; target mixture weight $\omega\in(0, 1)$; initial regularization power $\tau>0$; number of gradient steps per iteration $m$
        \STATE Set $\phi_{\targ, 1}^{(0)}\leftarrow\phi_1$ and $\phi_{\targ, 2}^{(0)}\leftarrow\phi_2$
        \STATE Initialize $k\leftarrow0$
        \WHILE{\textit{not done}}
            \STATE Observe state $s$ and take action $a\sim\pi_\theta(\cdot\mid s)$
            \STATE Observe next state $s'$, reward $r$, done signal $d$ and add $(s, a, s', r, d)$ to buffer $\mc{D}$\\ \hfill // \COMMENT{$d=1$ if $s'$ is a terminal state; otherwise, $d=0$}
            \STATE If $s'$ is a terminal state, reset the environment
            \IF{\textit{it's time to update}}
                \STATE Randomly sample a batch of transitions $\mc{B}=\Bp{(s, a, r, s', d)}\subseteq\mc{D}$
                \STATE Update regularization parameter by $$\tau^{\bkone}\leftarrow \tau^{\bk} - \lambda\nabla_{\tau}J(\tau, \theta^{\bk})$$\label{line:update_tau}
                \STATE Update q-value networks by $$\phi_i^{\bkone}\leftarrow \phi_i^{\bk} -\lambda_q\nabla_{\phi_i}J_q(\phi_i, \mc{B}, \phi_{\targ, 1}^{\bk}, \phi_{\targ, 2}^{\bk}, \theta^{\bk}, \tau^{\bkone}), \quad i=1, 2$$\label{line:update_Q}
                \STATE Set $\theta^{\bkone}_0\leftarrow \theta^{\bk}$
                \FOR{$j=1, \dots, m$}
                    \STATE Update policy network by {\color{blue} $$\theta^{\bkone}_{j}\leftarrow \theta^{\bkone}_{j-1} - \lambda_{\pi}\nabla_\theta \underset{\substack{s\sim\mc{B}\\ a\sim\pi^\theta(\cdot\mid s)}}{\E} \left.\Mp{\tau^{\bkone}\log\pi^\theta(a\mid s)-\min_{i=1, 2}q^{\phi^{\bkone}_i}(s, a)}\right|_{\theta=\theta^{\bkone}_{j-1}}$$}\label{line:update_pi}
                \ENDFOR
                \STATE Set $\theta^{\bkone}\leftarrow\theta^{\bkone}_m$
                \STATE Update target networks with $$\phi_{\targ, i}^{\bkone}\leftarrow \omega\phi_{\targ, i}^{\bk} + (1-\omega)\phi_i^{\bkone},\quad i=1, 2$$
                \STATE Update $k\leftarrow k+1$
            \ENDIF
        \ENDWHILE
    \end{algorithmic}
\end{algorithm}

Therefore, to implement \algoacronym-KL, we replace the update rule in line \ref{line:update_pi} of Algorithm \ref{algo:original_sac} by
{\color{blue} 
\begin{align*}
    \theta^{\bkone}_{j}\leftarrow \theta^{\bkone}_{j-1} - &\lambda_{\pi}\nabla_\theta \underset{\substack{s\sim\mc{B}\\ a\sim\pi^\theta_s}}{\E} \left[\tau^{\bkone}\log\pi^\theta(a\mid s) - (1-\beta)\tau^{\bkone}\log\pi^{\theta^{\bk}}(a\mid s)\right. \\
    & \qquad\qquad\qquad \left.\left. -\beta \min_{i=1, 2}q^{\phi^{\bkone}_i}(s, a)\right]\right|_{\theta=\theta^{\bkone}_{j-1}},
\end{align*}
}
where $\beta$ is a user-specified hyperparameter. Note that we use the standard reparameterization trick to compute the above gradient \citep{kingma2013auto}.

\subsubsection{AMPO}

To implement AMPO in \citet{alfano2023novel}, we will need to replace the update rule in line \ref{line:update_pi} of Algorithm \ref{algo:original_sac} by AMPO's loss in Eq. \eqref{equ:loss_apmo} in a more concrete form. That is, we have
\begin{align}
    \theta^{\bkone}\in&\argmin_{\theta}\E_{s\sim d^{\bk}_\rho}\Mp{\Norm{f^{\theta}_s - \Sp{\log\pi^{\bk}_s + \eta_k Q^{\bk}_s}}_2^2}\label{equ:ampo_var_1}\\
    =&\argmin_{\theta}\E_{s\sim d^{\bk}_\rho}\Mp{\Norm{f^{\theta}_s - (1-\eta_k \tau)\log\pi^{\bk}_s - \eta_k q_{\tau, s}^{\bk}}_2^2}.\nonumber
\end{align}
Here, $f^{\theta}$ should be the exponent of a Gaussian distribution. Therefore, if $g^{\theta}:\S\mapsto\R^{2n_{\A}}$ is the policy network, where $n_{\A}$ is the dimension of the action space. Then, we have
\begin{equation}
    \label{equ:gaussian_exponent}
    f^{\theta}(s, a)=-\sum_{i=1}^{n_{\A}}\frac{\Sp{g^{\theta}(s)_i}^2-2a_ig^{\theta}(s)_i}{2\Sp{g^{\theta}(s)_{n_{\A}+i}}^2}.
\end{equation}
Therefore, to implement AMPO, we replace the update rule in line \ref{line:update_pi} of Algorithm \ref{algo:original_sac} by
{\color{blue}
\begin{align*}
    \theta^{\bkone}_{j}\leftarrow \theta^{\bkone}_{j-1} - &\lambda_{\pi}\nabla_\theta \underset{\substack{s\sim\mc{B}\\ a\sim\mathsf{Unif}(\A)}}{\E} \left[\left(f^{\theta}(s, a) - (1-\eta \tau^{\bkone})\log\pi^{\theta^{\bk}}(a\mid s)\right.\right.\\
    &\qquad\qquad\qquad \left.\left.\left. - \eta \min_{i=1, 2}q^{\phi_i^{\bk}}(s, a)\right)^2\right]\right|_{\theta=\theta^{\bkone}_{j-1}},
\end{align*}
}
where $\eta$ is the mirror descent learning rate.

\subsubsection{MAMPO}
As discussed in Section \ref{sec:experiments}, MAMPO tries to optimize
$$\theta^{\bkone}\in\argmin_{\theta}\E_{s\sim d^{\bk}_\rho}\Mp{\Norm{f^{\theta}_s - \Sp{\pi^{\bk}_s + \eta_k Q^{\bk}_s}}_2^2},$$
where $f^{\theta}(s, a)$ is defined in Eq. \eqref{equ:gaussian_exponent}. Therefore, to implement MAMPO, we replace the update rule in line \ref{line:update_pi} of Algorithm \ref{algo:original_sac} by
{\color{blue}
\begin{align*}
    \theta^{\bkone}_{j}\leftarrow \theta^{\bkone}_{j-1} - & \lambda_{\pi}\nabla_\theta \underset{\substack{s\sim\mc{B}\\ a\sim\mathsf{Unif}(\A)}}{\E} \left[\left(f^{\theta}(s, a) - \pi^{\theta^{\bk}}(a\mid s) - \eta \min_{i=1, 2}q^{\phi_i^{\bk}}(s, a)\right.\right.\\
    &\qquad\qquad\qquad\left.\left.\left. + \eta\tau^{\bkone}\log\pi^{\theta^{\bk}}(a\mid s)\right)^2\right]\right|_{\theta=\theta^{\bkone}_{j-1}}.
\end{align*}
}

\subsection{Hyperparameter Settings}
We use the implementation of SAC from the \textit{Stable Baseline 3} under the MIT license \citep{raffin2021stable}. Then, we implement \algoacronym-KL, AMPO-KL, and MAMPO as modifications of its SAC's implementation. All model trainings were completed on 8 NVIDIA V100 GPUs in cluster.

We use SAC's default hyperparameters on all environments for both SAC and \algoacronym-KL, while AMPO-KL and MAMPO contain some tuning. Full hyperparameter details are provided in Table \ref{tab:hyper}.
\begin{table}[ht]
    \centering
    \caption{Hyperparameters of all algorithms}
    \label{tab:hyper}
    \begin{threeparttable}
        \begin{tabular}{c|cccc}
            \toprule
            Hyperparameter & SAC & \textbf{\algoacronym-KL} & AMPO & MAMPO \\
            \midrule
            Adam learning rate & $3\times 10^{-4}$ & $3\times 10^{-4}$ & $2\times 10^{-5}$ & $3\times 10^{-4}$ \\
            MD learning rate ($\eta$) & NA & NA & 1.0 & 1.0\\
            Entropy regularization ($\tau$) & auto\tnote{*} & auto\tnote{*} & 0 & 0\\
            \midrule
            Number of hidden layers & \multicolumn{4}{c}{2} \\
            Hidden layer size & \multicolumn{4}{c}{256} \\
            Batch size & \multicolumn{4}{c}{256} \\
            Discount factor ($\gamma$) & \multicolumn{4}{c}{0.99} \\
            Target mixture weight ($\omega$) & \multicolumn{4}{c}{0.005} \\
            Replay buffer size & \multicolumn{4}{c}{$1\times 10^6$} \\
            \bottomrule
        \end{tabular}
        \begin{tablenotes}
            \item[*] Being ``auto'' in entropy regularization means to use the update rule at line \ref{line:update_tau} of Algorithm \ref{algo:original_sac} to automatically adjust $\tau$.
        \end{tablenotes}
    \end{threeparttable}
\end{table}

Particularly for hyperparameter $\beta$ in \algoacronym-KL, we take different values for different tasks as shown in Table \ref{tab:beta}.
\begin{table}[ht]
    \centering
    \caption{Values of hyperparameter $\beta$ in \algoacronym-KL for different MuJoCo tasks.}
    \label{tab:beta}
    \begin{tabular}{c|cccc}
        \toprule
         Environments   &  HalfCheetah-v4 & Hopper-v4 & Walker2d-v4 & Ant-v4\\
        \midrule
        $\beta$  &  0.7 & 0.6 & 0.4 & 0.7\\
        \bottomrule
    \end{tabular}
\end{table}

\section{Additional Experiment Results on \AMPO}
\label{sec:additional_results}




In this section, we first introduce the original version of AMPO proposed in \citet{alfano2023novel}, which is slightly different from what we have in Eq. \eqref{equ:ampo_var_1}. Then, we present a partial record of our efforts in tuning AMPO, which shows the difficulty of using this algorithm in practical scenario. Nevertheless, we retain the possibility that our implementation of AMPO may not be the optimal.

\subsection{Variants of AMPO}
The original version of AMPO proposed in \citet{alfano2023novel} is given as
\begin{equation}
    \label{equ:ampo_var_2}
    \theta^{\bkone}\in\argmin_{\theta}\E_{s\sim d^{\bk}_\rho}\Mp{\Norm{f^{\theta}_s - \Sp{f^{\bk}_s + \eta_k Q^{\bk}_s}}_2^2}.
\end{equation}
While seemingly different from Eq. \eqref{equ:ampo_var_1}, these two are essentially the same from a theoretical perspective. To see this, as discussed in Example \ref{ex:ent-simplex}, when $\Phi$ is the negative entropy restricted on $\Delta(\A)$, we can freely take $\nabla\Phi(\pi)$ to be any vector in $\partial\Phi(\pi)$ while the corresponding Bregman divergence $D_{\Phi}$ is still well-defined. In particular, we have $\partial\Phi(\pi)=\Bp{\log\pi+c\bm{1}\mid c\in\R}$ with $\bm{1}=\matenv{1 & \cdots & 1}^\top$. As a result, since the difference between $f^{\theta}_s$ and $\log\pi^{\bk}_s$ is only an action-independent normalization constant, Eq. \eqref{equ:ampo_var_1} and Eq. \eqref{equ:ampo_var_2} are theoretically equivalent.\footnote{While \citet{alfano2023novel} claims to obtain Eq. \eqref{equ:ampo_var_2} by taking $\Phi$ to be the negative entropy on $\R^{\abs{\A}}_+$, this is not an appropriate argument because such a choice of $\Phi$ will enforce $\nabla\Phi(\pi)=\log\pi + \bm{1}$, excluding the freedom of choosing action-independent constant.}

Nevertheless, Eq. \eqref{equ:ampo_var_1} and Eq. \eqref{equ:ampo_var_2} may still be empirically different since the constant difference can still affect the $L_2$-loss minimization. Therefore, we consider and empirically compare these two different theoretically equivalent variants of AMPO-KL.\footnote{We use the variant in Eq. \eqref{equ:ampo_var_1} in all previous experiments.}

\subsection{Comparison between MAMPO and AMPO-KL}
Here, we provide a comparison between MAMPO and the two variants of AMPO-KL in Fig. \ref{fig:ampo_compare}, where both variants use the same set of hyperparameters as given in Table \ref{tab:hyper}.
\begin{figure*}[ht]
    \centering
    \includegraphics[width=\linewidth]{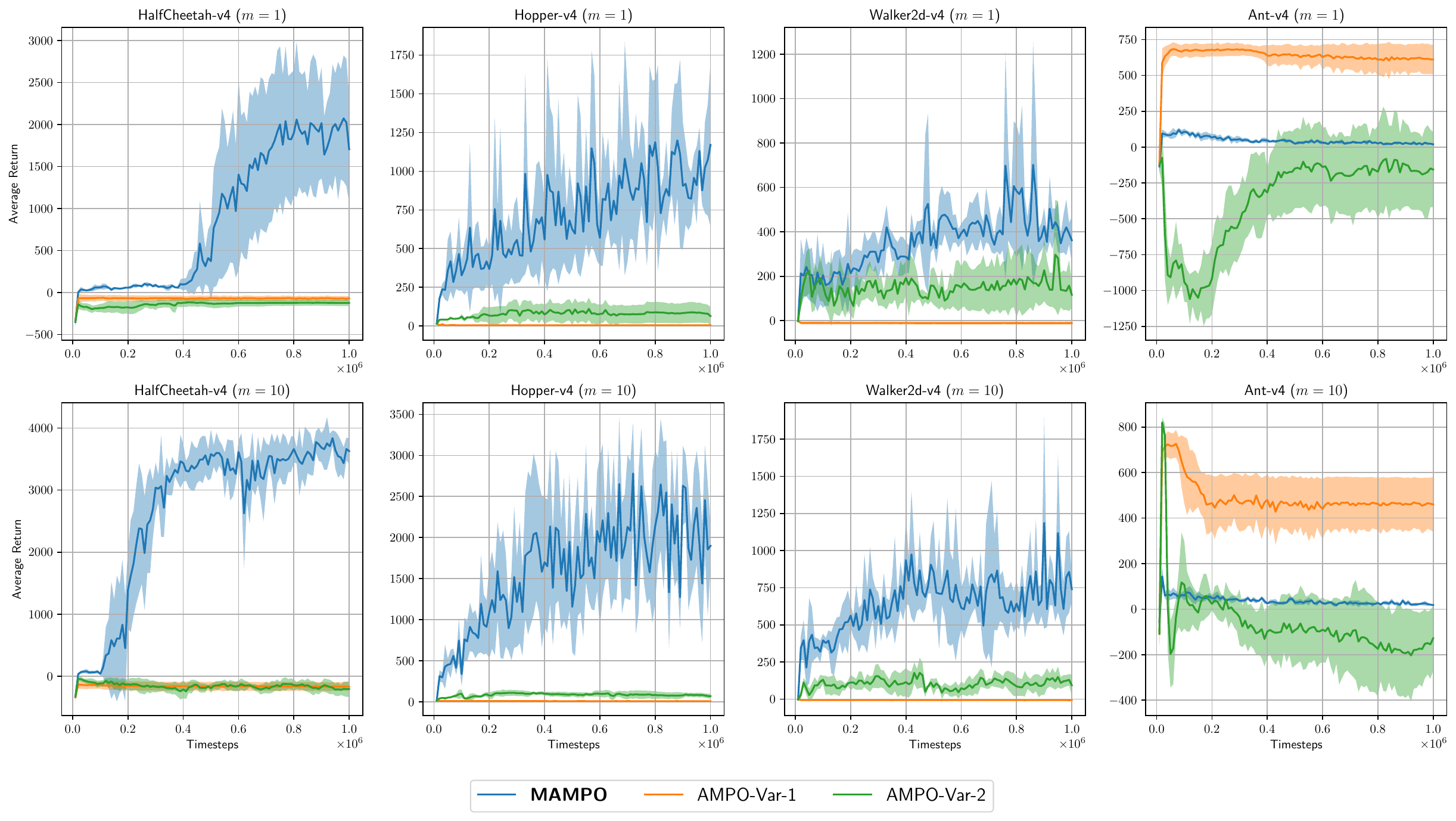}
    \caption{Comparison under $m=1$ and $m=10$ gradient steps per iteration between MAMPO and variants of AMPO-KL. Here, ``AMPO-Var-1'' refers to Eq. \eqref{equ:ampo_var_1} and ``AMPO-Var-2'' refers to Eq. \eqref{equ:ampo_var_2}. Each curve is averaged over 5 different random seeds and the shaded area represents the $95\%$ confidence interval.}
    \label{fig:ampo_compare}
\end{figure*}

We can see that both variants of AMPO-KL almost cannot learn anything non-trivial in all tasks.

\subsection{AMPO-KL under Different Hyperparameters}
Finally, we also provide a performance comparison of variants of AMPO-KL under different hyperparameter settings, where we only show the final-step performance under each setting, given in Table \ref{tab:hsearch1}, \ref{tab:hsearch2}, \ref{tab:hsearch3} and \ref{tab:hsearch4}. Nevertheless, we can easily see that AMPO-KL still cannot learn anything non-trivial under all of these settings.

\begin{table}[ht]
    \centering
    \caption{Final-step performance of AMPO-Var-1 (Eq. \eqref{equ:ampo_var_1}) in HalfCheetah-v4 with entropy regularization ($\tau=1.0$). Each data point is averaged over 3 different random seeds and $\pm$ represents the $95\%$ confidence interval.}
    \label{tab:hsearch1}
    \begin{tabular}{c|c|c|c|c}
        \toprule
         & $\mathrm{lr}=5\times 10^{-6}$ & $\mathrm{lr}=1\times 10^{-5}$ & $\mathrm{lr}=5\times 10^{-5}$ & $\mathrm{lr}=1\times 10^{-4}$ \\
        \midrule
        $\eta_k=0.1$ & $-94.08\pm 37.82$ & $-77.18\pm 16.73$ & $-3.33\pm 0.17$ & $-178.56\pm 41.57$ \\
        \hline
        $\eta_k=1$ & $-79.8\pm 47.96$ & $-61.83\pm 20.85$ & $-4.19\pm 0.49$ & $-14.93\pm 11.05$ \\
        \hline
        $\eta_k=10$ & $-27.83\pm 28.96$ & $-201.02\pm 71.98$ & $-8.37\pm 0.33$ & $-7.58\pm 0.73$ \\
        \hline
        $\eta_k=100$ & $-220.2\pm 124.88$ & $-210.46\pm 168.3$ & $-8.12\pm 0.35$ & $-7.36\pm 0.75$ \\
        \bottomrule
    \end{tabular}
\end{table}

\begin{table}[ht]
    \centering
    \caption{Final-step performance of AMPO-Var-1 (Eq. \eqref{equ:ampo_var_1}) in HalfCheetah-v4 without entropy regularization ($\tau=0$). Each data point is averaged over 3 different random seeds and $\pm$ represents the $95\%$ confidence interval.}
    \label{tab:hsearch2}
    \begin{tabular}{c|c|c|c|c}
        \toprule
         & $\mathrm{lr}=5\times 10^{-6}$ & $\mathrm{lr}=1\times 10^{-5}$ & $\mathrm{lr}=5\times 10^{-5}$ & $\mathrm{lr}=1\times 10^{-4}$ \\
        \midrule
        $\eta_k=0.1$ & $-131.27\pm 62.23$ & $-129.24\pm 51.24$ & $-123.17\pm 63.02$ & $-109.34\pm 84.59$ \\
        \hline
        $\eta_k=1$ & $-93.94\pm 46.46$ & $-95.82\pm 49.19$ & $-83.12\pm 59.42$ & $-97.82\pm 47.88$ \\
        \hline
        $\eta_k=10$ & $-50.57\pm 45.8$ & $-98.75\pm 24.9$ & $-81.51\pm 29.18$ & $-57.17\pm 39.36$ \\
        \hline
        $\eta_k=100$ & $-301.4\pm 128.66$ & $-295.53\pm 63.03$ & $-196.38\pm 136.29$ & $-255.11\pm 143.2$ \\
        \bottomrule
    \end{tabular}
\end{table}

\begin{table}[ht]
    \centering
    \caption{Final-step performance of AMPO-Var-2 (Eq. \eqref{equ:ampo_var_2}) in HalfCheetah-v4 with entropy regularization ($\tau=1.0$). Each data point is averaged over 3 different random seeds and $\pm$ represents the $95\%$ confidence interval.}
    \label{tab:hsearch3}
    \begin{tabular}{c|c|c|c|c}
        \toprule
         & $\mathrm{lr}=5\times 10^{-6}$ & $\mathrm{lr}=1\times 10^{-5}$ & $\mathrm{lr}=5\times 10^{-5}$ & $\mathrm{lr}=1\times 10^{-4}$ \\
        \midrule
        $\eta_k=0.1$ & $-3.56\pm 0.29$ & $-3.48\pm 0.52$ & $-3.39\pm 0.22$ & $-3.59\pm 0.22$ \\
        \hline
        $\eta_k=1$ & $-4.34\pm 0.47$ & $-4.24\pm 0.32$ & $-4.29\pm 0.26$ & $-4.22\pm 0.34$ \\
        \hline
        $\eta_k=10$ & $-8.38\pm 0.23$ & $-8.33\pm 0.12$ & $-8.4\pm 0.29$ & $-8.08\pm 0.57$ \\
        \hline
        $\eta_k=100$ & $41.59\pm 79.77$ & $-8.38\pm 0.1$ & $-8.33\pm 0.56$ & $-8.36\pm 0.52$ \\
        \bottomrule
    \end{tabular}
\end{table}

\begin{table}[ht]
    \centering
    \caption{Final-step performance of AMPO-Var-2 (Eq. \eqref{equ:ampo_var_2}) in HalfCheetah-v4 without entropy regularization ($\tau=0$). Each data point is averaged over 3 different random seeds and $\pm$ represents the $95\%$ confidence interval.}
    \label{tab:hsearch4}
    \begin{tabular}{c|c|c|c|c}
        \toprule
         & $\mathrm{lr}=5\times 10^{-6}$ & $\mathrm{lr}=1\times 10^{-5}$ & $\mathrm{lr}=5\times 10^{-5}$ & $\mathrm{lr}=1\times 10^{-4}$ \\
        \midrule
        $\eta_k=0.1$ & $-120.73\pm 34.97$ & $-178.62\pm 34.49$ & $-174.38\pm 67.63$ & $-144.4\pm 54.12$ \\
        \hline
        $\eta_k=1$ & $-87.45\pm 5.28$ & $-121.66\pm 56.37$ & $-129.73\pm 43.94$ & $-157.94\pm 49.55$ \\
        \hline
        $\eta_k=10$ & $-534.72\pm 205.25$ & $-237.98\pm 127.79$ & $-176.42\pm 235.9$ & $-199.77\pm 12.45$ \\
        \hline
        $\eta_k=100$ & $-341.7\pm 94.17$ & $139.05\pm 118.93$ & $-271.07\pm 149.01$ & $-411.67\pm 56.92$ \\
        \bottomrule
    \end{tabular}
\end{table}


\end{document}